\documentclass[11pt]{article} 
\usepackage[utf8]{inputenc} 
\usepackage[T1]{fontenc}    
\usepackage{hyperref}       
\usepackage{url}            
\usepackage{booktabs}       
\usepackage{amsfonts}       
\usepackage{nicefrac}       
\usepackage{microtype}      

\usepackage{amsmath}
\usepackage{amssymb}
\usepackage{amsthm}
\usepackage{pifont}
\usepackage{algorithm}
\usepackage{algpseudocode}
\usepackage[capitalize]{cleveref}
\usepackage{breqn}
\usepackage{xspace}
\usepackage{forloop}
\usepackage{multirow}
\usepackage{bbm}
\usepackage[pdftex,dvipsnames,table]{xcolor}
\usepackage[colorinlistoftodos,prependcaption,textsize=small]{todonotes}
\usepackage{graphicx}
\usepackage{paralist}
\usepackage{subcaption}
\usepackage{dsfont}
\usepackage{natbib}

\newtheorem{theorem}{Theorem}
\newtheorem{lemma}{Lemma}
\newtheorem{assump}{Assumption}

\crefname{assump}{Assumption}{Assumptions}

\newcommand{\tnabla}{\tilde{\nabla}}
\newcommand{\expect}{\mathbb{E}}
\newcommand{\constraint}{\mathcal{K}}
\newcommand{\one}{\mathbf{1}}
\newcommand{\downclose}{\mathcal{P}}
\newcommand{\domain}{\mathcal{X}}
\newcommand{\ie}{\emph{i.e.}}

\newcommand{\matroid}{\mathcal{I}}
\newcommand{\lb}{\underline{u}}

\newcommand{\Alg}{\texttt{Mono-Frank-Wolfe}\xspace}
\newcommand{\AlgB}{\texttt{Bandit-Frank-Wolfe}\xspace}
\newcommand{\AlgD}{\texttt{Responsive-Frank-Wolfe}\xspace}
\newcommand{\OCSM}{\texttt{OCSM}\xspace}
\newcommand{\BCSM}{\texttt{BCSM}\xspace}
\newcommand{\RBSM}{\texttt{RBSM}\xspace}
\newcommand{\BSM}{\texttt{BSM}\xspace}

\DeclareMathOperator*{\argmax}{arg\,max}

\DeclareMathOperator{\round}{round}
\DeclareMathOperator{\conv}{conv}
\DeclareMathOperator{\randomround}{RandRound}
\DeclareMathOperator{\losslessround}{LosslessRound}

\newcommand{\defcal}[1]{\expandafter\newcommand\csname 
c#1\endcsname{{\mathcal{#1}}}}
\newcommand{\defbb}[1]{\expandafter\newcommand\csname 
b#1\endcsname{{\mathbb{#1}}}}
\newcounter{calBbCounter}
\forLoop{1}{26}{calBbCounter}{
	\edef\letter{\Alph{calBbCounter}}
	\expandafter\defcal\letter
	\expandafter\defbb\letter
}

\begin{document}
\title{Online Continuous Submodular Maximization: From Full-Information to 
	Bandit Feedback}

\author{\\
	Mingrui Zhang\\
	Yale University\\
	\texttt{mingrui.zhang@yale.edu}
	\and\\
	Lin Chen \\ 
	Yale University \\
	\texttt{lin.chen@yale.edu}
	\and\\
	Hamed Hassani \\
	University of Pennsylvania\\
	\texttt{hassani@seas.upenn.edu}
	\and\\
	Amin Karbasi\\
	Yale University \\
	\texttt{amin.karbasi@yale.edu}\\
}
\maketitle

\begin{abstract}
In this paper, we propose three online algorithms for submodular maximization. The first one, \Alg, reduces the number of per-function gradient evaluations from $T^{1/2}$ \citep{Chen2018Online} and $T^{3/2}$ \citep{chen2018projection} to 1, 
and achieves a $(1-1/e)$-regret bound of $O(T^{4/5})$. 
The second one, \AlgB, is the first bandit algorithm for continuous 
DR-submodular maximization, which achieves a $(1-1/e)$-regret 
bound of $O(T^{8/9})$. Finally, we extend \AlgB to a bandit algorithm for 
discrete submodular 
maximization, \AlgD, which attains a $(1-1/e)$-regret bound of $O(T^{8/9})$ in 
the 
responsive bandit setting.
\end{abstract}
\newpage
\section{Introduction}
Submodularity naturally arises in a variety of disciplines, and has numerous 
applications in machine learning, including data summarization 
\citep{tschiatschek2014learning}, active and semi-supervised learning 
\citep{golovin2011adaptive, wei2015submodularity}, compressed sensing and 
structured sparsity \citep{bach2012optimization}, fairness in machine learning 
\citep{balkanski2015mechanisms}, mean-field inference in probabilistic models 
\citep{bian2018optimal}, and MAP inference in determinantal point processes 
(DPPs) \citep{kulesza2012determinantal}.

We say that a set function $f: 2^{\Omega} \to \mathbb{R}_{\geq 0}$ defined on a 
finite ground set $\Omega$ is \emph{submodular} if for every 
$A \subseteq B \subseteq \Omega$ and $x\in \Omega \setminus B$, we have 
$f(x|A)\ge f(x|B)$,
where $f(x|A)\triangleq f(A\cup\{x\})-f(A)$ is a discrete derivative~\citep{nemhauser1978analysis}. 
Continuous DR-submodular functions are the 
continuous analogue.  
Let $F:\mathcal{X}\to \mathbb{R}_{\ge 0}$ be a differentiable function defined 
on a box $\mathcal{X}\triangleq \prod_{i=1}^d \mathcal{X}_i$, where each 
$\mathcal{X}_i$ is a closed interval of $\mathbb{R}_{\ge 0}$. 
We say that $F$ is \emph{continuous DR-submodular} if for every $x,y\in 
\mathcal{X}$ that satisfy $x \leq y$ and every $i\in [d]\triangleq \{1,\dots, 
d\}$, 
we have $\frac{\partial F}{\partial x_i}(x) \geq \frac{\partial F}{\partial 
x_i}(y)$,
where $x\le y$ means $x_i\le y_i, \forall i\in [d]$~\citep{bian16guaranteed}. 

In this paper, we focus on online and bandit maximization of submodular set functions and continuous DR-submodular functions.
In contrast to offline optimization where the objective function is completely known beforehand, 
online optimization can be viewed as a two-player game between the player and 
the adversary in a sequential 
manner \citep{zinkevich2003online,shalev2007primal,hazan2012projection}. Let 
$\mathcal{F}$ be a family of real-valued functions.
The player wants to maximize a sequence of functions $F_1,\dots, F_T \in 
\mathcal{F}$ subject to a 
constraint set $\constraint$. The player has no \emph{a priori} knowledge of 
the functions, while the constraint set is known and we assume that it is a 
closed convex set in $\mathbb{R}^d$. The natural number $T$ 
is termed the \emph{horizon} of the online optimization problem. At the $t$-th 
iteration, without the knowledge of $F_t$, the player has to select a point 
$x_t \in \constraint$. After the player commits to this choice, the adversary 
selects a function $F_t \in \mathcal{F}$. The player 
receives a reward $F_t(x_t)$, 
observes the function $F_t$ determined by the adversary, and proceeds to the 
next iteration. In the more challenging bandit setting, even the function $F_t$ 
is unavailable to the player and the only observable information is the 
reward that the player 
receives \citep{flaxman2005online,agarwal2011stochastic,bubeck2016multi}.

The performance of the algorithm that the player uses to determine her choices $x_1, \dots, x_T$ is quantified by the regret, which is the gap between her accumulated reward and the reward of the best single choice in hindsight. To be precise, 
the regret is defined by $
  \max_{x\in \constraint} \sum_{t=1}^T F_t(x) - \sum_{t=1}^T F_t(x_t)$.
However, even in the offline scenario, it is shown that the maximization problem of a continuous DR-submodular function cannot be approximated within a factor of $(1-1/e+\epsilon)$ for any $\epsilon>0$ in polynomial time, unless $RP=NP$~\citep{bian16guaranteed}. Therefore, we consider the $(1-1/e)$-regret~\citep{streeter2009online,kakade2009playing,Chen2018Online}
\[
\mathcal{R}_{1-1/e, T} \triangleq(1-1/e) \max_{x\in \constraint} \sum_{t=1}^T 
F_t(x) - \sum_{t=1}^T F_t(x_t).
\]
For ease of notation, we write $\mathcal{R}_T$ for $\mathcal{R}_{1-1/e,T}$ throughout this paper.

In this paper, we study the following three problems:
\begin{compactitem}
\item \OCSM: the Online Continuous DR-Submodular Maximization problem, 
\item \BCSM: the Bandit Continuous DR-Submodular Maximization problem, and 
\item \RBSM: the Responsive Bandit Submodular Maximization problem.
\end{compactitem}

We note that although special cases of bandit submodular 
maximization problem (\BSM) were studied in 
\citep{streeter2009online,golovin14online}, the vanilla \BSM problem is still 
open for general monotone submodular functions under a matroid constraint. 
In \BSM, the objective functions $ 
f_1,\dots,f_T $ are submodular set functions defined on a common finite ground 
set $ \Omega $ and subject to a common constraint $ \cI $. For each 
function $ f_i $, the player has to select a subset $ X_i \in \cI $. 
Only after playing the subset $ X_i $, the reward $ 
f_i(X_i) $ is received and thereby observed. 

If the value of the corresponding multilinear extension\footnote{We formally 
define the multilinear extension of a submodular set function in 
\cref{sec:prelim}.} $F$ can be estimated by the submodular set function $f$, we 
may expect to solve the 
vanilla \BSM by invoking algorithms for continuous DR-submodular maximization. 
In this paper, however, we will show a hardness result that subject to some 
constraint $ \cI $, it is impossible to construct a one-point 
unbiased estimator of the multilinear extension $F$ based on the value of $f$, 
without knowing 
the information of $f$ in advance.
This result motivates the study of a slightly relaxed
setting termed the 
\emph{Responsive Bandit Submodular Maximization} problem 
(\RBSM). In \RBSM, at round $i$, if $ X_i\notin \cI $, the player is still 
allowed to play $ X_i$ and observe the function value $ f_i(X_i) $, but gets 
zero reward out of it. 

\OCSM was studied in \citep{Chen2018Online,chen2018projection}, where $T^{1/2}$ 
exact gradient evaluations or $T^{3/2}$ stochastic gradient evaluations are 
required per iteration ($T$ is the horizon).
Therefore, they cannot be extended to the bandit setting (\BCSM and \RBSM) 
where one single function evaluation per iteration is permitted. As a result, 
no known bandit algorithm attains a sublinear $(1-1/e)$-regret. 

In this paper, we first propose \Alg for \OCSM, which requires one 
stochastic gradient per function and still attains a $(1-1/e)$-regret bound of 
$O(T^{4/5})$. This is significant as it reduces 
the number of per-function gradient evaluations from 
$T^{3/2}$ to 1. Furthermore, it provides a feasible avenue to solving \BCSM and 
\RBSM. We then propose \AlgB and \AlgD that attain a $(1-1/e)$-regret bound of 
$O(T^{8/9})$ for \BCSM and \RBSM, respectively.
To the best of our knowledge, \AlgB and \AlgD are the first algorithms that 
attain a sublinear $(1-1/e)$-regret bound for \BCSM and \RBSM, respectively.

The performance of prior approaches and our proposed algorithms is summarized 
in \cref{tab:compare}. We also list 
further related works in \cref{app:relatedwork}.

\begin{table}[tbh]
\caption{Comparison of previous and our proposed algorithms.}\label{tab:compare}
\begin{center}
    \begin{tabular}{l  l  l  l  l }
    \toprule
    \multirow{ 2}{*}{Setting} & \multirow{ 2}{*}{Algorithm} & Stochastic   & \# of  grad.  & \multirow{ 2}{*}{$(1-1/e)$-regret}  \\ 
     & &    gradient & evaluations &   \\ \midrule
  \multirow{ 3}{*}{\OCSM}   & \texttt{Meta-FW}~\citep{Chen2018Online} & No & 
  $T^{1/2}$ & $O(\sqrt{T})$ \\ 
    & \texttt{VR-FW}~\citep{chen2018projection}   & Yes & $T^{3/2}$ & 
    $O(\sqrt{T})$ \\ 
 &  \cellcolor{gray!25}\texttt{Mono-FW} \textbf{(this work)}  & Yes & $1$ & $O(T^{4/5})$ \\ \midrule
  \BCSM & \cellcolor{gray!25}\texttt{Bandit-FW} \textbf{(this work)} & - & - & $O(T^{8/9})$\\ \midrule
  \RBSM & \cellcolor{gray!25}\texttt{Responsive-FW} \textbf{(this work)} & - & 
  - 
  & $O(T^{8/9})$\\
    \bottomrule
    \end{tabular}
\end{center}
\end{table}

\section{Preliminaries}\label{sec:prelim}
\paragraph{Monotonicity, Smoothness, and Directional Concavity Property} A submodular set function $f: 2^{\Omega} \to \mathbb{R}$ is called \emph{monotone} if for any two sets $A\subseteq B\subseteq \Omega$ we have $f(A)\le f(B)$. 

For two vectors $x$ and $ y$, we write $x\le y$ if $x_i\le y_i$ holds for every 
$i$. Let $F$ be a continuous DR-submodular function defined on $\domain$. We 
say that $F$ is \emph{monotone} if $F(x) \le F(y)$ for every $x,y\in \domain$ 
obeying $x\le y$. Additionally, $F$ is called  \emph{$L$-smooth} if for every 
$x,y \in \domain$ it holds that $\|\nabla F(x) - \nabla F(y) \| \leq L \|x-y 
\|$. Throughout the paper, we use the notation $\|\cdot \|$ for the Euclidean 
norm. An important implication of continuous DR-submodularity is  concavity 
along the non-negative 
directions~\citep{calinescu2011maximizing,bian16guaranteed}, \emph{i.e.}, for 
all $x \leq y$, we have $F(y) \leq F(x) + \langle \nabla F(x), y-x \rangle$.

\paragraph{Multilinear Extension} Given a submodular set function $f: 
2^{\Omega} \to \mathbb{R}_{\geq 0}$ defined on a finite ground set $\Omega$, 
its \emph{multilinear extension} is a continuous DR-submodular function $F: 
[0,1]^{|\Omega|} \to \mathbb{R}_{\geq 0}$ defined by $F(x) = \sum_{S \subseteq 
\Omega} f(S)\Pi_{i \in S}x_i \Pi_{j \notin S}(1-x_j)$, where $x_i$ is the 
$i$-th coordinate of $x$. Equivalently, for any vector $ x\in [0,1]^{|\Omega|}$ 
we have $F(x) = \expect_{S \sim x} [f(S)]$ where $ S\sim x $ means that  $ S $ 
is a 
random subset of $ \Omega $ such that every element $ i\in \Omega $ is 
contained in $ S $ independently with probability $ x_i $. 



  
\paragraph{Geometric Notations} The  $d$-dimensional unit ball is denoted by $B^d$, and  the $(d-1)$-dimensional unit sphere is denoted by $S^{d-1}$. Let $ \constraint $ be a bounded set. We define its diameter $D = \sup_{x,y \in \mathcal{K}} 
\|x-y \|$ and radius $R = \sup_{x \in \mathcal{K}} \| x\|$.  We say a set $\constraint$ has lower bound $\underline{u}$ if $\lb \in \constraint$, and $\forall x \in \constraint, x \geq \underline{u}$.

\section{One-shot Online Continuous DR-Submodular 
Maximization}\label{sec:online}
In this section, we propose \Alg, an online continuous DR-submodular 
maximization algorithm which only needs one gradient evaluation per function. 
This algorithm is the basis of the methods presented in the next section for 
the bandit setting. We also note that throughout this paper, $\nabla F$ denotes 
the exact gradient for $F$, while $\tnabla F$ denotes the stochastic gradient.

We begin by reviewing the Frank-Wolfe (FW) \citep{Frank1956algorithm, 
jaggi2013revisiting} method for maximizing
monotone continuous DR-submodular functions in the offline 
setting~\citep{bian16guaranteed}, where we have one single objective function 
$F$. Assuming that we have access to the exact gradient 
$\nabla F$, the FW 
method is an iterative procedure that 
starts from the initial point $ x^{(1)} = 0 $, and at the $ k $-th 
iteration, solves a linear optimization problem
\begin{equation}\label{eq:linear_opt}
v^{(k)}\gets \argmax_{v\in 
	\constraint} \langle 
v, \nabla F(x^{(k)})\rangle
\end{equation}
which is used to update 
$x^{(k+1)} \gets x^{(k)} + \eta_k v^{(k)}$, 
where $ \eta_k $ is the step size. 

We aim to extend the FW method to the online setting. Inspired by the FW 
update above, to get high rewards for each objective function $F_t$, we 
start from $x_t^{(1)}=0$, update $x_t^{(k+1)} = x_t^{(k)} + \eta_k 
v_t^{(k)}$ for multiple iterations (let $K$ denote the number of iterations), 
then play the last iterate 
$x_t^{(K+1)}$ for $F_t$. 
To obtain the point $x_t^{(K+1)}$ which we play, we need to solve the 
linear program 
\cref{eq:linear_opt} and thus  
get $v_t^{(k)}$, where we have to know the gradient in advance.
However, in the online setting, we 
can only observe the stochastic gradient $\tnabla F_t$ \emph{after} we play 
some 
point for $F_t$.  
So the key issue is to obtain the vector $v_t^{(k)}$ which at least 
approximately maximizes $\langle \cdot, \nabla F_t(x_t^{(k)})\rangle$, 
\emph{before} 
we play some point for $F_t$. 

To do so, we use $K$ no-regret online linear maximization oracles 
$\{\mathcal{E}^{(k)}\}, k\in[K]$, and let $v_t^{(k)}$ be the output vector of 
$\mathcal{E}^{(k)}$ at round $t$. Once we update $x_t^{(k+1)}$ by $v_t^{(k)}$ 
for all $k\in[K]$, 
and play $x_t^{(K+1)}$ for $F_t$, we can observe 
$\tnabla F_t(x_t^{(k)})$ and iteratively construct 
$d_t^{(k)}=(1-\rho_k)d_t^{(k-1)} + 
\rho_k \tnabla 
F_t(x_t^{(k)})$, an estimation of $\nabla 
F_t(x_t^{(k)})$ with a lower variance than $\tnabla F_t(x_t^{(k)})$ 
\citep{mokhtari2017conditional,mokhtari2018stochastic} for all $k\in [K]$. Then 
we set 
$\langle \cdot, d_t^{(k)}\rangle$ as the objective function for oracle 
$\mathcal{E}^{(k)}$ at round $t$. Thanks to the no-regret property of 
$\mathcal{E}^{(k)}$, $v_t^{(k)}$, which is obtained before we 
play some point for $F_t$ and observe the gradient, 
approximately maximizes $\langle \cdot, d_t^{(k)}\rangle$, thus also 
approximately maximizes $\langle \cdot, \nabla F_t(x_t^{(k)})\rangle$. 

This approach was first proposed in 
\citep{Chen2018Online,chen2018projection}, where stochastic gradients at 
$K=T^{3/2}$ points (\emph{i.e.}, $\{x_t^{(k)}\},k\in[K]$) are 
required for each function $F_t$. To carry this general idea into the 
one-shot 
setting where 
we can only access one gradient per function, we need the following blocking 
procedure.
We divide the upcoming objective functions $ F_1,\dots, F_T $ into 
$Q$ equisized blocks of size $K$ (so $T=QK$). For the $q$-th block, we first 
set $x_q^{(1)}=0$, update $x_q^{(k+1)} = x_q^{(k)} + \eta_k v_q^{(k)}$, and 
play the same point $x_q = x_q^{(K+1)}$ for all the functions 
$F_{(q-1)K+1},\dots,F_{qK}$. The reason why we play the same point $x_q$ will 
be explained later. 
We also define 
the average function in the $q$-th block as $\bar{F}_q \triangleq 
\frac{1}{K}\sum_{k=1}^{K} F_{(q-1)K+k}$. In order to reduce the required number 
of gradients per function, the key idea is to view the average functions 
$\bar{F}_1,\dots,\bar{F}_{Q}$ as \emph{virtual} objective functions.

Precisely, in the $q$-th block, let $(t_{q,1},\dots,t_{q,K})$ be a random 
permutation 
of the indices $\{(q-1)K+1,\dots,qK \}$. \emph{After} we update all the 
$x_q^{(k)}$, for each $F_t$, we play $x_q$ and find the 
corresponding $k'$ 
such that $t = t_{q,k'}$, then observe $\tnabla F_t$ (\ie, $\tnabla 
F_{t_{q,k'}}$)
at $x_q^{(k')}$. Thus we can obtain 
$\tnabla F_{t_{q,k}}(x_q^{(k)})$ for all $k \in [K]$. Since 
$t_{q,k}$ is a random variable such that $\bE[F_{t_{q,k}}] = \bar{F}_q$, 
$\tnabla F_{t_{q,k}}(x_q^{(k)})$ is also an estimation of $\nabla 
\bar{F}_q(x_q^{(k)})$, which holds 
for all 
$k\in[K]$. As a result, with only one gradient evaluation per function 
$F_{t_{q,k}}$, we can obtain stochastic gradients of the virtual objective 
function $\bar{F}_q$ at $K$ points. In this way, the required number of 
per-function gradient evaluations is reduced from $K$
to 1 successfully. 

Note that since we play $y_t = x_q$ for each $F_t$ in the $q$-th block, 
the regret w.r.t.\ the original objective functions and that w.r.t. the average 
functions satisfy that 
\[
(1-1/e)\max_{x\in \constraint} \sum_{t=1}^T F_t(x)-\sum_{t=1}^T F_t(y_t) = 
K\left[ (1-1/e)\max_{x\in \constraint} \sum_{q=1}^{Q} 
\bar{F}_q(x)-\sum_{t=1}^{Q} 
\bar{F}_q(x_q) \right]\,,
\] 
which makes it possible to view the functions $\bar{F}_q$ as \emph{virtual} 
objective functions in the regret analysis.
Moreover, we iteratively
construct $d_q^{(k)}=(1-\rho_k)d_q^{(k-1)} + 
\rho_k \tnabla F_{t_{q,k}}(x_q^{(k)})$ as an estimation of $\nabla 
F_{t_{q,k}}(x_q^{(k)})$, thus also an estimation of $\nabla 
\bar{F}_q(x_q^{(k)})$. So $v_q^{(k)}$, the output of 
$\mathcal{E}^{(k)}$, approximately maximizes $\langle \cdot, \nabla 
\bar{F}_q(x_q^{(k)})\rangle$. Inspired by the offline FW method, playing 
$x_q=x_q^{(K+1)}$, the last iterate in the FW 
procedure, may obtain high rewards for $\bar{F}_q$.
As a result, we play the 
same point $x_q$ in the $q$-th block.

We also note that once $ t_{q,1},\dots,t_{q,k} $ are revealed, conditioned on 
the knowledge, the 
expectation of 
$ F_{t_{q,k+1}}$ is no longer 
the average function $\bar{F}_q$ but the residual average function 
$\bar{F}_{q,k}(x) = \frac{1}{K-k}\sum_{i=k+1}^K F_{t_{q,i}}(x)$. 
As more indices $t_{q,k}$ are revealed, 
$\bar{F}_{q,k}$ 
becomes increasingly different from 
$\bar{F}_q$, which makes the observed gradient $\tnabla 
F_{t_{q,k+1}}(x_q^{(k+1)})$ 
not a good estimation of $\nabla \bar{F}_q(x_q^{(k+1)})$ any more.
As a result, although we use the averaging technique  (the update of 
$d_q^{(k)}$)  
as in~\citep{mokhtari2017conditional,mokhtari2018stochastic} for variance 
reduction, a \emph{completely different} gradient error analysis is required.
In \cref{lem:delta_final} (\cref{app:one_shot}), we 
establish that the squared error of $d_q^{(k)}$ 
exhibits an inverted 
bell-shaped tendency; \ie, the squared error is large at the initial and final 
stages and is small at the intermediate stage. 

We present our proposed \Alg algorithm in~\cref{alg:one-shot}. 

\begin{algorithm}[htb]
	\caption{\Alg}
	\begin{algorithmic}[1]
		\Require constraint set $\constraint$, horizon $T$, 
		block size $K$, online linear maximization oracles on 
		$\constraint$: 
		$\mathcal{E}^{(1)},\cdots,\mathcal{E}^{(K)}$, 
		step sizes $\rho_k \in (0,1), 
		\eta_k \in (0,1)$, number of blocks $Q=T/K$
		\Ensure $y_1, y_2, \dots$
		\For{$q=1,2,\dots, Q$}
		\State $d_q^{(0)}\gets 0$, $x_q^{(1)}\gets 0 $
		\State For $ k=1,2,\dots, K$, let $v_q^{(k)} \in \constraint$ be the 
		output of $\mathcal{E}^{(k)}$ in round $q$, $x_q^{(k+1)}\gets x_q^{(k)} 
		+ \eta_k v_q^{(k)}.$ Set $ x_q\gets 
		x_q^{(K+1)}$\label{ln:update_online}
		\State Let $(t_{q,1}, \dots, t_{q,K})$ be a random permutation of $ 
		\{(q-1)K+1,\dots,qK \} $
		\label{ln:permutation}
		\State For $ t = (q-1)K+1,\dots,qK$, play $y_t=x_q$ and 
		obtain the reward $F_t(y_t)$; find the corresponding 
		$k'\in[K]$ such that $t=t_{q,k'}$, observe $\tnabla 
		F_t(x_q^{(k')})$, \emph{i.e.}, $\tnabla 
		F_{t_{q,k'}}(x_q^{(k')})$\label{ln:play}	
		\State For $ k=1,2,\dots, K$, $d_q^{(k)}\gets 
		(1-\rho_k)d_q^{(k-1)}+\rho_k 
		\tnabla F_{t_{q,k}}(x_q^{(k)})$, compute $\langle v_q^{(k)}, 
		d_q^{(k)}\rangle$ as reward for $\mathcal{E}^{(k)}$, and feed back 
		$d_q^{(k)}$ to $\mathcal{E}^{(k)}$ \label{ln:variance_reduction}
		\EndFor
	\end{algorithmic}
	\label{alg:one-shot}
\end{algorithm}

We will show that \Alg achieves a $(1-1/e)$-regret bound of $O(T^{4/5})$. 
In order to prove this result, we first make the following assumptions on the 
constraint set $\constraint$, the 
objective functions $F_t$, the stochastic gradient $\tnabla F_t$, and the 
online linear maximization oracles. 

\begin{assump}
	\label{assump_on_K}
	The constraint set $\mathcal{K}$ is a convex and compact set that contains 
	$ 0 
	$.
\end{assump}

\begin{assump}
	\label{assump_on_f}
	Every objective function $F_t$ is monotone, continuous  DR-Submodular, 
	$L_1$-Lipschitz, and $L_2$-smooth. 
\end{assump}

\begin{assump}
	\label{assump_on_estimate}
	The stochastic gradient $ \tilde{\nabla}F_t(x) $ is unbiased, \ie, $\expect 
	[ \tilde{\nabla}F_t(x)] = \nabla F_t(x)$. Additionally, it has a uniformly 
	bounded 
	norm $\| \tnabla F_t(x) \| \leq M_0$ and a uniformly bounded variance 
	$\expect [\| 
	\nabla F_t(x) - \tilde{\nabla}F_t(x)\|^2] \leq \sigma_0^2$ for every 
	$x \in \constraint$ and objective function $F_t$.
\end{assump}

\begin{assump}\label{assump_on_oracle}
	For the online linear maximization oracles, the regret at 
	horizon $t$ (denoted by $ \mathcal{R}_t^{\mathcal{E}^{(i)}} $) satisfies
	$\mathcal{R}_t^{\mathcal{E}^{(i)}} \leq C \sqrt{t}, \forall i \in [K]$, 
	where $C>0$ is a constant.
\end{assump}

Note that there exist online linear maximization oracles $\mathcal{E}^{(i)}$ 
with regret 
$\mathcal{R}_t^{\mathcal{E}^{(i)}} \leq C \sqrt{t}, \forall i \in [K]$ for any 
horizon $ t $ (for example, the online gradient descent 
\citep{zinkevich2003online}). Therefore, \cref{assump_on_oracle} is fulfilled.

\begin{theorem}[Proof in \cref{app:one_shot}]
	\label{thm:one_shot}
	Under \cref{assump_on_oracle,assump_on_K,assump_on_f,assump_on_estimate}, 
	if we set $K = T^{3/5}, \eta_k = \frac{1}{K}, \rho_k = 
	\frac{2}{(k+3)^{2/3}}$ when $1 \leq k 
	\leq K/2+1$, and $\rho_k = \frac{1.5}{(K-k+2)^{2/3}}$ when $K/2+2 \leq k 
	\leq 
	K$, where we assume that $K$ is even for simplicity, 
	then $y_t \in \constraint, \forall t$, and the expected $(1-1/e)$-regret of 
	\cref{alg:one-shot} is at most
	$$\expect[\mathcal{R}_T] \leq (N+C+D^2)T^{4/5} + \frac{L_2D^2}{2}T^{2/5}, $$
	where 
	$N = \max \{5^{2/3}(L_1+M_0)^2, 4(L_1^2+\sigma_0^2) + 32G, 
	2.25(L_1^2+\sigma_0^2)+7G/3\}, G=(L_2R+2L_1)^2.$
\end{theorem}

\section{Bandit Continuous DR-Submodular Maximization}
In this section, we present the first bandit algorithm for 
continuous DR-submodular maximization, \AlgB, which attains a $(1-1/e)$-regret 
bound of $O(T^{8/9})$. We begin by explaining the one-point gradient 
estimator \citep{flaxman2005online}, which is crucial to the proposed bandit 
algorithm. The proposed algorithm and main results are illustrated in 
\cref{sec:ee_tradeoff}.

\subsection{One-Point Gradient Estimator}
\label{sec:smoothing_trick}
Given a function $ F $, we define
its $\delta$-smoothed version $\hat{F_\delta}(x) \triangleq \expect_{v \sim 
B^d}[F(x + \delta v)]$,
where $v\sim B^d$ denotes that $v$ is drawn uniformly at random from the unit 
ball 
$B^d$. Thus the function $ F $ is averaged over a ball of radius $ \delta $.
It can be easily verified that if $F$ is monotone, continuous DR-submodular, 
$L_1$-Lipschitz, and $L_2$-smooth, then so is $\hat{F}_\delta$, and for all 
$x$ we have $| \hat{F}_\delta(x) - F(x) | \leq L_1 \delta$ 
(\cref{lem:smooth_property} in \cref{app:smooth_property}).
So the $ \delta $-smoothed version $\hat{F}_\delta$ is indeed an 
approximation of $F$. A maximizer 
of $\hat{F}_\delta$ also maximizes $F$ approximately.

More importantly, the 
gradient of the smoothed function $\hat{F}_\delta$ admits a one-point unbiased 
estimator \citep{flaxman2005online,hazan2016introduction}: $\nabla 
\hat{F}_\delta (x) =  \expect_{u \sim S^{d-1}} \left[ \frac{d}{\delta}F(x + 
\delta u)u\right]$, 
where $u\sim S^{d-1}$ denotes that $u$ is drawn uniformly at random from the 
unit sphere 
$S^{d-1}$.
Thus the player can 
estimate the gradient of the smoothed function at point $ x $ by playing the 
random point $ x+\delta u $ for the original function $F$. So usually, we can 
extend a one-shot online algorithm to the bandit setting by replacing the 
observed stochastic gradients with the one-point gradient estimations.
%

In our setting, however, we cannot use the one-point gradient estimator 
directly. When the point $x$ is close to the boundary of the constraint set 
$\constraint$, the 
point $x + \delta u$ may fall outside of $\constraint$. To address this issue, 
we introduce the notion of \emph{$\delta$-interior}. A set is said to be a $ 
\delta $-interior of $\constraint$ if it is a \emph{subset} of
\begin{equation*}
\text{int}_\delta(\constraint) = \{x \in \constraint | \inf_{s \in \partial 
	\constraint} d(x,s) \ge \delta\}\,,    
\end{equation*}
where $ d(\cdot,\cdot) $ denotes the Euclidean distance.

In other words, $\constraint'$ is a $\delta$-interior of $\constraint$ if it 
holds for every $x \in \constraint'$ that  $B(x, \delta) \subseteq 
\constraint$ (\cref{subfig:example} in \cref{app:shrink_down_closed}). 
We note that there can be infinitely many $\delta$-interiors of $\constraint$.
In the sequel,  $\constraint'$ will denote 
the $\delta$-interior that we consider. We also define the discrepancy 
between $\constraint$ and $\constraint'$ by 
\begin{equation*}
d(\constraint,\constraint') = \sup_{x \in \constraint} d(x,\constraint'),   
\end{equation*}
which is the supremum of the distances between points in $\constraint$ and the 
set $\constraint'$. The distance $ d(x,\constraint') $ is given by $ \inf_{y\in 
\constraint'} d(x,y) $.

By 
definition, 
every point $x \in \constraint'$ satisfies
$x+ \delta u \in \constraint$, which enables us to 
use the one-point gradient 
estimator on $\constraint'$. Moreover, if 
every $F_t$ is Lipschitz and $d(\constraint,\constraint')$ is small, we can 
approximate the optimal total reward on $\constraint$ ($\max_{x\in \constraint} 
\sum_{t=1}^T 
F_t(x)$) by that on $\constraint'$ ($\max_{x\in \constraint'} \sum_{t=1}^T 
F_t(x)$),
and thereby 
obtain the regret bound subject to the original constraint set $\constraint$, 
by running bandit algorithms on $\constraint'$.

We also note that if the constraint set $\constraint$ satisfies 
\cref{assump_on_K} and 
is  down-closed (\emph{e.g.}, a matroid polytope), for sufficiently small 
$\delta$,
we can \emph{construct} $\constraint'$, a down-closed $\delta$-interior 
of $\constraint$, 
with $d(\constraint,\constraint')$ sufficiently small (actually it is a linear 
function of $\delta$). Recall that a set 
$\downclose$ is down-closed if it has a lower bound $\underline{u}$ such that 
(1) $\forall y \in \downclose, \underline{u} \leq y$; and (2) $\forall y \in 
\downclose, x \in \mathbb{R}^d, \underline{u} \leq x \leq y \implies x \in 
\downclose$~\citep{bian16guaranteed}. 

We first define $B_{\geq 0}^d = B^d \cap \mathbb{R}_{\geq 0}^d$ 
and make the following assumption\footnote{This assumption is an analogue of 
the 
		assumption $rB^d \subseteq \constraint \subseteq RB^d$ in 
		\citep{flaxman2005online}.}: 
\begin{assump}
	\label{assump_on_discrete_bandit_K}
	There exists a positive number $r$ such that $rB^d_{\geq 0} \subseteq 
	\constraint$. 
\end{assump}
To construct $\constraint'$, for sufficiently small $ \delta $ such that $ 
\delta<\frac{r}{\sqrt{d}+1} $, we first set 
$\alpha = \frac{(\sqrt{d}+1)\delta}{r} < 1$, and
shrink $\constraint$ by a factor of $(1-\alpha)$ to obtain $\constraint_\alpha 
= (1-\alpha)\constraint$. Then we translate the shrunk set $\constraint_\alpha$ 
by $\delta 
\one$ (\cref{subfig:construction} in \cref{app:shrink_down_closed}). In other 
words, the set that we finally obtain is 
\begin{equation*}
\constraint' = 
\constraint_\alpha + 
\delta \one = (1-\alpha)\constraint + \delta \one.  \end{equation*}
In \cref{lem:shrink_down_closed}, we establish that $ \constraint' $ is 
indeed a $\delta$-interior of $\constraint$ and deduce a linear bound for $ 
d(\constraint,\constraint') $.

\begin{lemma}[Proof in \cref{app:shrink_down_closed}]
	\label{lem:shrink_down_closed}
	We assume \cref{assump_on_K,assump_on_discrete_bandit_K} and also assume 
	that 
	$\constraint$ is down-closed and 
	that $\delta$ is sufficiently small such that $\alpha = 
	\frac{(\sqrt{d}+1)\delta}{r} < 1$. The set 
	$\constraint'=(1-\alpha)\constraint+\delta\one$  
	is convex and compact. Moreover, $\constraint'$ is a down-closed 
	$\delta$-interior of $\constraint$ and satisfies $d(\constraint, 
	\constraint') 
	\leq [\sqrt{d}(\frac{R}{r}+1) + \frac{R}{r}] \delta$.
\end{lemma}

\subsection{No-\texorpdfstring{$(1-1/e)$}{(1-1/e)}-Regret Biphasic Bandit 
Algorithm}
\label{sec:ee_tradeoff}
Our proposed bandit algorithm is based on the online algorithm \Alg in 
\cref{sec:online}. Precisely, we want to replace the stochastic gradients in 
\cref{alg:one-shot} with the one-point gradient estimators, and run the 
modified 
algorithm on $\constraint'$, a proper $\delta$-interior of the constraint set 
$\constraint$. 
Note that the one-point estimator requires that the point at which we estimate 
the gradient (\emph{i.e.}, $x$) must be identical to the point that we 
play (\emph{i.e.}, $x + \delta u$), if we ignore the random $\delta u$. In 
\cref{alg:one-shot}, however, we 
play point $x_q$ 
but obtain estimated gradient at other points 
$x_q^{(k')}$ (\cref{ln:play}). 
This suggests that \cref{alg:one-shot} cannot be extended to the bandit setting 
via the 
one-point gradient estimator directly.

To circumvent this limitation, we propose a biphasic approach that categorizes 
the plays into the exploration and exploitation phases. To motivate this 
biphasic method, recall that in \cref{alg:one-shot}, we need to play $ x_q $  
to gain high rewards (exploitation), whilst we observe $\tnabla 
F_t(x_q^{(k')})$ to obtain gradient information (exploration). So in our 
biphasic approach, we expend a large portion of plays on 
exploitation (play $x_q$, so we can still get high rewards) and a small portion 
of plays on exploring the gradient (play $x_q^{(k')}$ to get one-point gradient 
estimators, so we can still obtain sufficient information). 

%

To be precise, we divide the $T$ objective functions into $Q$ equisized blocks 
of size $L$, where $L = T/Q$. Each block is subdivided into two phases.
As shown in \cref{alg:bandit},
we randomly choose $K \ll L$ functions for exploration (\cref{ln:exploration}) 
and use the remaining 
$(L-K)$ functions for exploitation (\cref{ln:exploitation}). 

We describe our algorithm formally in \cref{alg:bandit}. We also note that for 
a general constraint set $\constraint$ with a proper $\delta$-interior 
$\constraint'$ such that $d(\constraint,\constraint')\leq c_1\delta^\gamma$, 
\emph{\cref{thm:bandit}} (\cref{app:general_K}) shows a $(1-1/e)$-regret bound 
of
$O(T^{\frac{3+5\min \{1,\gamma\}}{3+6\min\{1,\gamma\}}})$. Moreover, with 
\cref{lem:shrink_down_closed}, this result can be 
extended to down-closed constraint sets $\constraint$, as shown in 
\cref{thm:downclose_bandit}.


\begin{algorithm}[htb]
	\caption{\AlgB}
	\begin{algorithmic}[1]
		\Require smoothing radius $\delta$, $\delta$-interior $\constraint'$ 
		with 
		lower bound $\lb$, horizon $T$, block size $L$,
		the number of exploration steps per block $K$, online linear 
		maximization oracles 
		on $\constraint'$: $\mathcal{E}^{(1)},\cdots,\mathcal{E}^{(K)}$, 
		step sizes $\rho_k \in (0,1), \eta_k \in (0,1)$, the number of blocks 
		$Q=T/L$
		\Ensure $y_1, y_2, \dots$
		\For{$q=1,2,\dots, Q$}
		\State $d_q^{(0)}\gets 0$, $x_q^{(1)}\gets \underline{u}$
		\State For $ k=1,2,\dots, K$, let $v_q^{(k)} \in \constraint'$ be the 
		output of $\mathcal{E}^{(k)}$ in round $q$, $x_q^{(k+1)}\gets x_q^{(k)} 
		+ \eta_k (v_q^{(k)} - \underline{u}).$ Set $ x_q\gets x_q^{(K+1)}$
		\State Let $(t_{q,1}, \dots, t_{q,L})$ be a random permutation of 
		$\{(q-1)L+1,\cdots, qL\}$
		\For{$t=(q-1)L+1,\cdots, qL$} 
		\State If $t \in \{ t_{q,1}, \cdots, t_{q,K} \}$, find the 
		corresponding $k'\in[K]$ 
		such that $t=t_{q,k'}$, play $y_t = y_{t_{q,k'}} = x_q^{(k')} + \delta 
		u_{q,k'}$ for 
		$F_t$ (\ie, $F_{t_{q,k'}}$), where $u_{q,k'}\sim S^{d-1}$ 
		\label{ln:exploration} \Comment{\emph{Exploration}}
		\State If $ t \in \{(q-1)L+1,\cdots, qL\}
		\setminus \{ t_{q,1}, \cdots, 
		t_{q,K} \} $, play $y_t = x_q$ for $F_t$ 
		\label{ln:exploitation} 
		\Comment{\emph{Exploitation}}
		\EndFor
		\State For $k=1,2,\dots, K$, $g_{q,k} \gets \frac{d}{\delta} 
		F_{t_{q,k}}(y_{t_{q,k}})u_{q,k}$, $d_q^{(k)}\gets 
		(1-\rho_k)d_q^{(k-1)}+\rho_k g_{q,k} $, compute $\langle v_q^{(k)}, 
		d_q^{(k)}\rangle$ as reward for 
		$\mathcal{E}^{(k)}$, and feed back $d_q^{(k)}$ to $\mathcal{E}^{(k)}$
		\EndFor
	\end{algorithmic}
	\label{alg:bandit}
\end{algorithm}


\begin{assump}
	\label{bandit_assump_on_f}
	Every objective function $F_t$ satisfies that 
	$\sup_{x \in \constraint} | F_t(x) 	|  \leq M_1$. 
	
\end{assump}

\begin{theorem}[Proof in \cref{app:donw-closed_K}] \label{thm:downclose_bandit}
	We assume 
	\cref{assump_on_oracle,assump_on_f,assump_on_K,bandit_assump_on_f,assump_on_discrete_bandit_K},
	 and also assume that $\constraint$ is down-closed. 
	If we generate $\constraint'$ as
	in \cref{lem:shrink_down_closed}, and set $\delta = 
	\frac{r}{\sqrt{d}+2}T^{-\frac{1}{9}}, 
	L = T^{\frac{7}{9}}, K = T^{\frac{2}{3}}, \eta_k = \frac{1}{K}, 
	\rho_k=\frac{2}{(k+2)^{2/3}}$, then $y_t \in \constraint, \forall t$, and 
	the expected $(1-1/e)$-regret of 
	\cref{alg:bandit} is at most
	\begin{equation*}
	\begin{split}
	\expect[\mathcal{R}_T] 
	\leq& N T^{\frac{8}{9}} + 
	\frac{3r[2L_1^2+(3L_2R+2L_1)^2]}{4^{1/3}(\sqrt{d}+2)}T^{\frac{2}{3}} + 
	\frac{L_2D^2}{2}T^{\frac{1}{3}},
	\end{split}    
	\end{equation*}
	where $N=\frac{(1-1/e)r}{\sqrt{d}+2}[\sqrt{d}(\frac{R}{r}+1)+\frac{R}{r}] 
	L_1 + \frac{(2-1/e)r}{\sqrt{d}+2}L_1+ 
	2M_1+\frac{3\cdot4^{1/6}(\sqrt{d}+2)d^2M_1^2}{r}+\frac{3(\sqrt{d}+2)D^2}{4r}+C$.
\end{theorem}

\section{Bandit Submodular Set Maximization}
In this section we aim to solve the problem of bandit submodular set 
maximization by 
lifting it to the continuous domain.
Let objective functions $f_1,\cdots,f_T : 2^{\Omega} \to \bR_{\geq 0}$ be a 
sequence of  monotone 
submodular set 
functions defined on a common ground set $ \Omega = \{1,\dots,d\} $. We also 
let $ 
\cI $ denote the matroid constraint, and $ \constraint $ be the matroid 
polytope of $\matroid$, \emph{i.e.}, $\constraint = \conv\{\one_I: I \in 
\matroid\} \subseteq [0,1]^d$ \citep{calinescu2011maximizing}, where $ \conv $ 
denotes the convex hull. 

\subsection{An Impossibility Result}\label{sec:hardness}
A natural idea  is that at each round $t$, we apply \AlgB, the continuous 
algorithm in \cref{sec:ee_tradeoff}, on $F_t$ subject to $\constraint$, where 
$F_t$ is the multilinear extension 
of the discrete objective function $f_t$. Then we get a 
fractional 
solution $y_t\in \constraint$, round it to a set $Y_t \in \mathcal{I}$, and 
play $Y_t$ for 
$f_t$.

For the exploitation phase,
we will use a lossless rounding scheme such that $f_t(Y_t) \geq F_t(y_t)$, so 
we will not get lower rewards after the rounding. 
Instances of such a lossless rounding scheme include  
pipage rounding~\citep{ageev2004pipage, calinescu2011maximizing} and the 
contention resolution scheme~\citep{vondrak2011submodular}.

In the exploration phase, we need to
use the reward $f_t(Y_t)$ to obtain an unbiased gradient estimator of the 
smoothed version of 
$F_t$. As the one-point estimator $\frac{d}{\delta}F(x 
+ \delta u)u$ in \cref{alg:bandit} is unbiased, 
we require the (random) rounding scheme $ \round_\cI:[0,1]^d \to \cI $ 
to satisfy the 
following unbiasedness condition
\begin{equation}\label{eq:sampling_scheme}
\expect[f(\round_\cI(x))] = F(x),\quad \forall x\in [0,1]^d
\end{equation}
for any submodular set function $ f $ on the ground set $ \Omega $ and its 
multilinear extension $ F $.


Since we have no \emph{a priori} knowledge of the objective function $f_{t}$ 
before playing a subset for it, such a rounding scheme $\round_\cI$ should 
\emph{not} depend on the function choice 
$f$. In other words, we need to find an \emph{independent} $\round_\cI$ such 
that 
\cref{eq:sampling_scheme} holds for any submodular function $f$ defined on 
$\Omega$. 

We first review the random rounding scheme $\randomround:[0,1]^d \to \cI $
\begin{equation}
\label{eq:sample}
\begin{cases}
i \in \randomround(x) &  \text{with probability } x_i\,;\\
i \notin \randomround(x) &  \text{with probability } 1- x_i\,.
\end{cases}
\end{equation}
In other words, each element $ i\in \Omega $ is included with an independent 
probability $ x_i $, where $x_i$ is the $i$-th coordinate of $x$.
$ \randomround $ satisfies the unbiasedness requirement 
\cref{eq:sampling_scheme}.
However, its range is $ 2^\Omega $ in general, so the 
rounded set may fall outside of $\matroid$. In fact, 
as shown in \cref{lem:sampling_scheme}, there exists a  matroid $\matroid$ for 
which we cannot 
find a proper unbiased rounding scheme whose range is contained in $ \cI $.

\begin{lemma}[Proof in \cref{app:sampling_scheme}]\label{lem:sampling_scheme}
	There exists a matroid $\matroid$ for which there is no   
	rounding scheme 
	$\round: 
	[0,1]^d \to \matroid$ whose construction does not 
	depend on the function $ f $ and which
	satisfies \cref{eq:sampling_scheme} 
	for any submodular set function $ f $. 
\end{lemma}


\subsection{Responsive Bandit Algorithm}
The impossibility result~\cref{lem:sampling_scheme} shows that the one-point 
estimator may be incapable of solving the general \BSM problem. As a result, 
we study a slightly relaxed setting termed  the responsive bandit submodular 
maximization problem (\RBSM). Let $ X_t $ be the subset that we play at the $ t 
$-th round. 
The only 
difference between the responsive bandit setting and the vanilla bandit setting 
is that in the responsive setting, if $X_t \notin \matroid$,
we can still observe the function value $f_t(X_t)$ 
as feedback, while the received reward at round $t$ is $ 0 $ 
(since the subset that we play violates the 
constraint $ \matroid $). In other words, the environment is 
always responsive to the player's decisions, no matter whether $ X_t $ is in $ 
\matroid $
or not. 

We note that the \RBSM problem has broad applications in both 
theory and practice. In theory, \RBSM can be regarded as a relaxation of \BSM, 
which helps us to better understand the nature of \BSM. In practice, the 
responsive model (not only for submodular maximization or bandit) has 
potentially many applications when a decision cannot be committed, while we can 
still get the potential outcome of the decision as feedback. 
For example, suppose that we have a replenishable inventory of items where 
customers arrive (in an online fashion) with a utility function unknown to us. 
We need to
allocate a collection of items to each customer, and the goal is to maximize the total 
utility (reward) of all the customers. We may use a partition matroid  to model 
diversity (in terms of category, time, \emph{etc}). In the \RBSM model, we 
cannot allocate the collection of items which violates the constraint to the 
customer, but we can use it as a questionnaire, and the customer will tell us 
the potential utility if she received those items. The feedback will help us to 
make better decisions in the future. Similar examples include portfolio 
selection when the investment choice is too risky, \emph{i.e.}, violates  the 
recommended constraint set, we may stop trading and thus get no reward on that 
trading period, but at the same time observe the potential reward if we 
invested in that way.

Now, we turn to propose our algorithm. As discussed in 
\cref{sec:hardness}, we want to solve the problem of bandit 
submodular set maximization by applying \cref{alg:bandit} on the multilinear 
extensions $F_t$ with different rounding schemes. Precisely, in the 
\emph{responsive} setting, we use the $ \randomround $ \cref{eq:sample} in the 
exploration phase to guarantee that we can always obtain unbiased gradient 
estimators, and use a lossless rounding scheme $ \losslessround $ in the 
exploitation phase to receive high rewards. 
We present \AlgD in \cref{alg:discrete_bandit}, and show that it achieves a 
$(1-1/e)$-regret bound of $O(T^{8/9})$.


\begin{algorithm}[htb]
	\caption{\AlgD}
	\begin{algorithmic}[1]
		\Require matroid constraint $\mathcal{I}$, matroid polytope 
		$\constraint$, smoothing radius $\delta$, $\delta$-interior 
		$\constraint'$ 
		with lower bound $\lb$, horizon $T$, block size $L$, 
		the number of exploration steps per block $K$, online linear 
		maximization oracles 
		on $\constraint'$: $\mathcal{E}^{(1)},\cdots,\mathcal{E}^{(K)}$, 
		steps sizes $\rho_k \in (0,1), \eta_k \in (0,1)$, the number of blocks 
		$Q=T/L$
		\Ensure $Y_1, Y_2, \dots$
		\For{$q=1,2,\dots, Q$}
		\State $d_q^{(0)}\gets 0$, $x_q^{(1)}\gets \lb $
		\State For $ k=1,2,\dots, K$, let $v_q^{(k)} \in \constraint'$ be the 
		output of $\mathcal{E}^{(k)}$ in round $q$, $x_q^{(k+1)}\gets x_q^{(k)} 
		+ \eta_k (v_q^{(k)} - \lb).$ Set $x_q\gets x_q^{(K+1)}$
		\State Let $(t_{q,1}, \dots, t_{q,L})$ be a random permutation of $ 
		\{(q-1)L+1,\cdots, qL\}  $
		\For{$t=(q-1)L+1,\cdots, qL$}
		\State If $t \in \{ t_{q,1}, \cdots, t_{q,K} \}$, find the 
		corresponding $k' \in [K]$ such that $t=t_{q,k'}$, play $Y_t = 
		Y_{t_{q,k'}} = \randomround(y_{t_{q,k'}})$ for $f_t$ (\ie, 
		$f_{t_{q,k'}}$), where $y_{t_{q,k'}} = x_q^{(k')} + \delta u_{q,k'}, 
		u_{q,k'} \sim S^{d-1}$. 
		If $Y_{t} \in 
		\matroid$, get reward $f_{t}(Y_{t})$; otherwise, get 
		reward 0. 
		\Comment{\emph{Exploration}}
		\State If $ t \in \{(q-1)L+1,\cdots, qL\} \setminus \{ t_{q,1}, 
		\cdots, 
		t_{q,K} \} $, play $Y_t = \losslessround(y_t)$ for $f_t$, where $y_{t} 
		= 
		x_q$ \Comment{\emph{Exploitation}}
		\EndFor
		\State For $ k=1,2,\dots, K$, $g_{q,k} \gets \frac{d}{\delta} 
		f_{t_{q,k}}(Y_{t_{q,k}})u_{q,k}$, $d_q^{(k)}\gets 
		(1-\rho_k)d_q^{(k-1)}+\rho_k g_{q,k} $, compute $\langle v_q^{(k)}, 
		d_q^{(k)}\rangle$ as reward for 
		$\mathcal{E}^{(k)}$, and feed back $d_q^{(k)}$ to $\mathcal{E}^{(k)}$  
		\EndFor
	\end{algorithmic}
	\label{alg:discrete_bandit}
\end{algorithm}

\begin{assump}
	\label{assump_on_discrete_bandit_f}
	Every objective function $f_t$ is monotone submodular with $\sup_{X 
	\subseteq \Omega} |f_t(X)| \le M_1$.
\end{assump}

\begin{theorem}[Proof in \cref{app:discrete_bandit}]
	\label{thm:discrete_bandit}
	Under 
	\cref{assump_on_oracle,assump_on_discrete_bandit_K,assump_on_discrete_bandit_f},
	if we generate 
	$\constraint'$ as in \cref{lem:shrink_down_closed}, and set $\delta = 
	\frac{r}{\sqrt{d}+2}T^{-\frac{1}{9}}, 
	L = T^{\frac{7}{9}}, 
	K = T^{\frac{2}{3}}, \eta_k = \frac{1}{K}, \rho_k=\frac{2}{(k+2)^{2/3}}$, 
	then 
	in the responsive setting, the expected $(1-1/e)$-regret of 
	\cref{alg:discrete_bandit} is at most
	\begin{equation*}
	\begin{split}
	\expect[\mathcal{R}_T] 
	\leq& N T^{\frac{8}{9}} + 
	\frac{3r[2L_1^2+(3\sqrt{d}L_2+2L_1)^2]}{4^{1/3}(\sqrt{d}+2)}T^{\frac{2}{3}} 
	+ \frac{L_2d}{2}T^{\frac{1}{3}},
	\end{split}    
	\end{equation*}
	where $N=\frac{(1-1/e)r}{\sqrt{d}+2}[\frac{d}{r}+\sqrt{d}(1+\frac{1}{r})] 
	L_1 + \frac{(2-1/e)r}{\sqrt{d}+2}L_1+ 
	3M_1+\frac{3\cdot4^{2/3}(\sqrt{d}+2)d^2M_1^2}{r}+\frac{3(\sqrt{d}+2)d}{4r}+C$,
	 $L_1 = 2M_1\sqrt{d}, L_2 = 4M_1\sqrt{d(d-1)}$.
\end{theorem}

\section{Conclusion}
In this paper, by proposing a series of novel methods including the 
\emph{blocking procedure} and the \emph{permutation methods}, we  
developed  \Alg  for the \OCSM problem, which requires 
only 
one stochastic gradient evaluation per function and still achieves a 
$(1-1/e)$-regret bound of $O(T^{4/5})$. We then introduced the 
\emph{biphasic method} and the notion of $\delta$-\emph{interior}, to extend 
\Alg to 
\AlgB for the \BCSM problem. Finally, we introduced the responsive model and 
the corresponding \AlgD Algorithm for the \RBSM problem. We proved  that both 
\AlgB and  \AlgD
attain a $(1-1/e)$-regret bound of $O(T^{8/9})$.


\subsubsection*{Acknowledgments}
This work is partially supported by the Google PhD Fellowship, NSF 
(IIS-1845032), ONR (N00014-19-1-2406) and AFOSR (FA9550-18-1-0160). 
We would   like to thank Marko~Mitrovic for his valuable 
comments 
and Zheng~Wei for help preparing some of the illustrations.

\bibliographystyle{plainnat}
\bibliography{main}

\newpage
\appendix


\section{Further Related Work}\label{app:relatedwork}

The framework of online convex optimization (OCO) dates back to 
\citep{zinkevich2003online}, where 
a regret bound of $O(\sqrt{T})$ was attained. 
The regret bound was improved to $\log(T)$ for strongly convex losses in 
\citep{hazan2007logarithmic}. 
The RFTL algorithm was proposed independently in ~\citep{shalev2007online, 
shalev2007primal}. 
The projection-free algorithm Online Conditional Gradient was proposed in 
~\citep{hazan2012projection, hazan2016introduction}. 
The model of Bandit Convex Optimization (BCO) was 
introduced in \citep{flaxman2005online}, and followed by plenty of works
~\citep{dani2008price,agarwal2011stochastic,bubeck2012regret, 
bubeck2016multi}. 
Various regret bounds were achieved by adding extra assumptions (\emph{e.g.}, 
strong convexity) in 
\citep{kleinberg2005nearly,agarwal2010optimal,saha2011improved,hazan2014bandit,bubeck2015bandit,dekel2015bandit,hazan2016optimal,bubeck2017kernel}.
The first computationally efficient projection-free BCO algorithm was 
proposed in \citep{chen2018projection2}. For 
strongly convex and smooth losses, a lower bound of 
$\Omega(\sqrt{T})$ for regret was proved in \citep{shamir2013complexity}. 
Bandit linear optimization 
was studied in ~\citep{abernethy2008competing,awerbuch2008online, 
bubeck2012towards}.
Interested readers are referred to \cite{bubeck2012regret} for a survey on BCO.

\citet{bach2015submodular} derived connections between continuous 
submodularity and convexity. \citet{bian16guaranteed} studied the offline 
continuous DR-submodular maximization and proposed a variant of the 
Frank-Wolfe 
algorithm to achieve the tight $(1-1/e)$ approximation ratio. 
In the online setting, maximization of submodular set functions was studied 
in~\citep{streeter2009online,golovin14online}. Adaptive submodular bandit 
maximization was analyzed in \citep{gabillon2013adaptive}. The
linear submodular bandit problems were studied in 
\citep{yue2011linear,yu2016linear}. 


\section{Proof of \cref{thm:one_shot}}\label{app:one_shot}
\begin{proof}
	Since $y_t = x_q = x_q^{(K+1)}$, which is a convex combination of 
	$v_q^{(1)}, v_q^{(2)}, \cdots, v_q^{(K)}$, and $v_q^{(k)} \in \constraint, 
	\forall k \in [K]$, we have $y_t \in \constraint$. 
	Then we proceed to prove the theorem.
	
	The key idea of \cref{alg:one-shot} is to use the average function of a 
	bunch of functions in certain group (\emph{e.g.}, the block) to represent 
	the functions. Note the regret is calculated by the sum of all the reward 
	functions, and the sum of average functions is exactly the sum of all the 
	functions divided by the block size, so we can use the average function to 
	analyze the regret.
	
	Let 
	$$\bar{F}_{q,k}(x) = \frac{\sum_{i=k+1}^K F_{t_{q,i}}(x)}{K-k}, k \in \{0, 
	1, \cdots, K-1 \}$$
	denotes the average function of the remaining $(K-k)$ functions after round 
	$k$ in the $q$-th block. Recall that $(t_{q,1}, \dots, t_{q,K})$ is a 
	random permutation of $ ((q-1)K, qK]\cap \mathbb{Z}$, thus 
	$\bar{F}_{q,k}(x)$ is a random function. Also, by definition, we have the 
	expected regret
	\begin{equation}
	\label{eq:regret_for_average}
	\expect [\sum_{t=1}^T (1-1/e)F_t(x^*) - F_t(x_q)] = \expect [ \sum_{q=1}^Q 
	K[(1-1/e)\bar{F}_{q,0}(x^*) - \bar{F}_{q,0}(x_q)]], 
	\end{equation}
	where $x^* = \argmax_{x \in \constraint} \sum_{t=1}^T F_t(x).$ We also note 
	that on the left hand side of \cref{eq:regret_for_average}, $q$ is actually 
	a 
	function of $t$. Specifically , $q$ is the index of the block which 
	contains 
	$F_t$.
	
	\begin{lemma}[Eq.(9) in \citep{chen2018projection}]
		\label{lem:icml}
		If $F_t$ is monotone continuous DR-submodular and $L_2$-smooth, 
		$x_t^{(k+1)} = x_t^{(k)}+ 1/K \cdot v_t^{(k)}$ for $k \in [K]$, then 
		\begin{equation*}
		\begin{split}
		F_t(x^*) -F_t(x_t^{(k+1)}) \leq & (1 - 1/K)[F_t(x^*) - F_t(x_t^{(k)})] 
		\\
		&-\frac{1}{K}[-\frac{1}{2\beta^{(k)}} \| \nabla F_t(x_t^{(k)}) - 
		d_t^{(k)} \|^2 - \frac{\beta^{(k)}D^2}{2} + \langle d_{t}^{(k)}, 
		v_t^{(k)} - x^* \rangle] + \frac{L_2D^2}{2K^2},   
		\end{split}
		\end{equation*}
		where $\{\beta^{(k)}\}$ is a sequence of positive parameters to be 
		determined.
	\end{lemma}

	\begin{lemma}
		\label{lem:regret_decomp}
		If $F_t$ is monotone continuous DR-submodular and $L_2$-smooth for all 
		$t$, $x_q^{(k+1)} = x_q^{(k)}+ 1/K \cdot v_q^{(k)}$ for $k \in [K]$, 
		and $x_q = x_q^{(K+1)}$, then we have
		\begin{equation*}
		\begin{split}
		\expect[(1-1/e)\bar{F}_{q,0}(x^*) - \bar{F}_{q,0}(x_q)] \leq & \expect[ 
		\frac{1}{K}\sum_{k=1}^K[\frac{1}{2\beta^{(k)}} \Delta_q^{(k)} + 
		\frac{\beta^{(k)}D^2}{2}]] + \frac{L_2D^2}{2K} \\
		& + 1/K \sum_{k=1}^K (1-1/K)^{K-k} \expect [\langle d_{q}^{(k)}, x^* -  
		v_q^{(k)} \rangle ],
		\end{split}    
		\end{equation*}
		where $\Delta_q^{(k)} = \|\nabla \bar{F}_{q,k-1}(x_q^{(k)}) - d_q^{(k)} 
		\|^2$.
	\end{lemma}
	
	\begin{proof}[Proof of \cref{lem:regret_decomp}]
		Since $F_t$ is monotone continuous DR-Submodular and $L_2$-smooth, then 
		so is $\bar{F}_{q,k-1}$. By \cref{lem:icml}, we have
		\begin{equation}
		\label{eq:iteration}
		\begin{split}
		\expect [ \bar{F}_{q,0}(x^*) -\bar{F}_{q,0}(x_q^{(k+1)}) ] 
		=& \expect [ \bar{F}_{q,k-1}(x^*) -\bar{F}_{q,k-1}(x_q^{(k+1)}) ] \\
		\leq& \expect [ (1 - 1/K)[\bar{F}_{q,k-1}(x^*) - 
		\bar{F}_{q,k-1}(x_q^{(k)})] + \frac{L_2D^2}{2K^2}\\
		& -\frac{1}{K}[-\frac{1}{2\beta^{(k)}} \| \nabla 
		\bar{F}_{q,k-1}(x_q^{(k)}) - d_q^{(k)} \|^2 - \frac{\beta^{(k)}D^2}{2} 
		+ \langle d_{q}^{(k)}, v_q^{(k)} - x^* \rangle]].
		\end{split}
		\end{equation}
		
		Note that $\expect \left[ \bar{F}_{q,k-1}(x^*) 
		-\bar{F}_{q,k-1}(x_q^{(k)})\right] = \expect[\bar{F}_{q,k-2}(x^*) 
		-\bar{F}_{q,k-2}(x_q^{(k)})]$, so we can apply \cref{eq:iteration} 
		recursively 
		for $k \in \{1, 2, \cdots, K\}$, and get
		\begin{equation*}
		\begin{split}
		\expect[\bar{F}_{q,0}(x^*) -\bar{F}_{q,0}(x_q) ] \leq & \expect 
		[(1-1/K)^K[\bar{F}_{q,0}(x^*) -\bar{F}_{q,0}(x_q^{(1)})] + 
		\frac{1}{K}\sum_{k=1}^K[\frac{1}{2\beta^{(k)}} \Delta_q^{(k)} + 
		\frac{\beta^{(k)}D^2}{2}]] \\
		& + \frac{L_2D^2}{2K} + 1/K \sum_{k=1}^K (1-1/K)^{K-k} \expect [\langle 
		d_{q}^{(k)}, x^* -  v_q^{(k)} \rangle ],
		\end{split}    
		\end{equation*}
		where $\Delta_q^{(k)} = \|\nabla \bar{F}_{q,k-1}(x_q^{(k)}) - d_q^{(k)} 
		\|^2$.
		
		Recall that $\bar{F}_{q,0}(x_q^{(1)}) = \bar{F}_{q,0}(0) \geq 0$ and 
		$(1-1/K)^K \leq 1/e, \forall K \geq 1$, so we have
		\begin{equation*}
		\label{eq:single_func} 
		\begin{split}
		\expect[(1-1/e)\bar{F}_{q,0}(x^*) - \bar{F}_{q,0}(x_q)] \leq & \expect[ 
		\frac{1}{K}\sum_{k=1}^K[\frac{1}{2\beta^{(k)}} \Delta_q^{(k)} + 
		\frac{\beta^{(k)}D^2}{2}]] + \frac{L_2D^2}{2K} \\
		& + 1/K \sum_{k=1}^K (1-1/K)^{K-k} \expect [\langle d_{q}^{(k)}, x^* -  
		v_q^{(k)} \rangle ].
		\end{split}
		\end{equation*}
	\end{proof}
	
	Combine \cref{eq:regret_for_average} and \cref{lem:regret_decomp}, we have 
	that the expected regret of \cref{alg:one-shot} satisfies:
	\begin{equation*}
	\begin{split}
	\expect[\mathcal{R}_T] &= \expect [\sum_{t=1}^T (1-1/e)F_t(x^*) - F_t(x_q)] 
	\\
	&= \expect [ \sum_{q=1}^Q K[(1-1/e)\bar{F}_{q,0}(x^*) - 
	\bar{F}_{q,0}(x_q)]]  \\
	&\leq \expect[ \sum_{q=1}^Q [\sum_{k=1}^K[\frac{1}{2\beta^{(k)}} 
	\Delta_q^{(k)} + \frac{\beta^{(k)}D^2}{2}]+ \frac{L_2D^2}{2}]] + 
	\sum_{q=1}^{Q} \sum_{k=1}^K (1-1/K)^{K-k} \expect [\langle d_{q}^{(k)}, x^* 
	-  v_q^{(k)} \rangle ]   \\
	& = \expect[\sum_{q=1}^Q\sum_{k=1}^K \frac{1}{2\beta^{(k)}} \Delta_q^{(k)} 
	+ \frac{D^2}{2}Q\sum_{k=1}^K \beta^{(k)}] + \frac{L_2D^2}{2}Q \\
	& \qquad + \sum_{k=1}^K  (1-1/K)^{K-k} \expect[ \sum_{q=1}^Q  \langle 
	d_{q}^{(k)}, x^* -  v_q^{(k)} \rangle ].
	\end{split}    
	\end{equation*}
	
	Since $v_q^{(k)}$ is the output of the online linear maximization oracle 
	$\mathcal{E}^{(k)}$ at round $q$, we have
	$$ \sum_{q=1}^Q \langle d_{q}^{(k)}, x^* -  v_q^{(k)} \rangle \leq 
	\mathcal{R}_Q^{\mathcal{E}}, $$ 
	and thus we have 
	\begin{equation*}
	\sum_{k=1}^K  (1-1/K)^{K-k} \expect[ \sum_{q=1}^Q \langle d_{q}^{(k)}, x^* 
	-  v_q^{(k)} \rangle ] \leq  \sum_{k=1}^K 1 \cdot 
	\mathcal{R}_Q^{\mathcal{E}} = K \mathcal{R}_Q^{\mathcal{E}}.  
	\end{equation*}
	
	Therefore, 
	\begin{equation}
	\label{eq:regret}
	\expect[\mathcal{R}_T] \leq \expect[\sum_{q=1}^Q\sum_{k=1}^K 
	\frac{1}{2\beta^{(k)}} \Delta_q^{(k)}] + \frac{D^2}{2}Q\sum_{k=1}^K 
	\beta^{(k)} + K\mathcal{R}_Q^{\mathcal{E}} + \frac{L_2D^2}{2}Q.
	\end{equation}
	
	Note $\mathcal{R}_Q^{\mathcal{E}}$ is the regret of oracle $\mathcal{E}$ at 
	horizon $Q$, which is of order $O(\sqrt{Q})$,
	so in order to get an upper bound for the expected regret of 
	\cref{alg:one-shot}, the key is to bound $\expect[\Delta_q^{(k)}]$. 
	
	\begin{lemma}
		\label{lem:delta_decomp} 
		Under the setting of \cref{thm:one_shot}, we have 
		\begin{equation*}
		\begin{split}
		\expect[\Delta_q^{(k)}] &\leq \rho_k^2 \sigma^2 + (1-\rho_k)^2 
		\expect[\Delta_q^{(k-1)}] + (1-\rho_k)^2 \frac{G}{(K-k+2)^2} \\
		& \qquad + (1-\rho_k)^2\left[ \frac{G}{\alpha_k (K-k+2)^2} + 
		\alpha_k\expect[\Delta _q^{(k-1)}] \right]
		\end{split}   
		\end{equation*}
		where $\{\alpha_k\}$ is a sequence of positive parameters to be 
		determined, $\sigma^2 = L_1^2+\sigma_0^2$, and $G=(L_2R+2L_1)^2$.
	\end{lemma}
	
	\begin{proof}[Proof of \cref{lem:delta_decomp}]
		By the definition of $d_q^{(k)}$, we have
		\begin{equation}
		\label{eq:delta}
		\begin{split}
		\Delta_q^{(k)} &= \| \nabla \bar{F}_{q,k-1}(x_q^{(k)}) - (1 - 
		\rho_k)d_q^{(k-1)} - \rho_k \tnabla F_{t_{q,k}}(x_q^{(k)}) \|^2 \\
		&= \| \rho_k[\nabla \bar{F}_{q,k-1}(x_q^{(k)}) - \tnabla 
		F_{t_{q,k}}(x_q^{(k)})] + (1 - \rho_k)[\nabla 
		\bar{F}_{q,k-1}(x_q^{(k)}) - \nabla \bar{F}_{q,k-2}(x_q^{(k-1)})] \\
		& \qquad + (1 - \rho_k)[\nabla \bar{F}_{q,k-2}(x_q^{(k-1)}) - 
		d_q^{(k-1)}]  \|^2 \\
		&= \rho_k^2 \| \nabla \bar{F}_{q,k-1}(x_q^{(k)}) - \tnabla 
		F_{t_{q,k}}(x_q^{(k)}) \|^2 + (1 - \rho_k)^2 \Delta_q^{(k-1)} \\
		& \qquad + (1 - \rho_k)^2 \| \nabla \bar{F}_{q,k-1}(x_q^{(k)}) - \nabla 
		\bar{F}_{q,k-2}(x_q^{(k-1)}) \|^2 \\
		& \qquad + 2\rho_k(1-\rho_k) \langle \nabla \bar{F}_{q,k-1}(x_q^{(k)}) 
		- \tnabla F_{t_{q,k}}(x_q^{(k)}), \nabla \bar{F}_{q,k-1}(x_q^{(k)}) - 
		\nabla \bar{F}_{q,k-2}(x_q^{(k-1)}) \rangle \\
		& \qquad + 2\rho_k(1-\rho_k) \langle \nabla \bar{F}_{q,k-1}(x_q^{(k)}) 
		- \tnabla F_{t_{q,k}}(x_q^{(k)}), \nabla \bar{F}_{q,k-2}(x_q^{(k-1)}) - 
		d_q^{(k-1)} \rangle \\
		& \qquad + 2(1-\rho_k)^2 \langle \nabla \bar{F}_{q,k-1}(x_q^{(k)}) - 
		\nabla \bar{F}_{q,k-2}(x_q^{(k-1)}) , \nabla 
		\bar{F}_{q,k-2}(x_q^{(k-1)}) - d_q^{(k-1)} \rangle.
		\end{split}
		\end{equation}
		
		For further analysis, we first denote $\mathcal{F}_{q,k}$ to be the 
		$\sigma$-field generated by $t_{q,1}, t_{q,2}, \cdots, t_{q,k}$. Then 
		by law of iterated expectations, 
		\begin{equation}
		\label{eq:decomp_11}
		\begin{split}
		&\expect [\| \nabla \bar{F}_{q,k-1}(x_q^{(k)}) - \tnabla 
		F_{t_{q,k}}(x_q^{(k)}) \|^2] \\
		=& \expect[\expect[\| \nabla \bar{F}_{q,k-1}(x_q^{(k)}) - \tnabla 
		F_{t_{q,k}}(x_q^{(k)}) \|^2  |\mathcal{F}_{q,k-1}]]  \\
		=& \expect[\expect[\| \nabla \bar{F}_{q,k-1}(x_q^{(k)}) - \nabla 
		F_{t_{q,k}}(x_q^{(k)}) \|^2  + \|\nabla F_{t_{q,k}}(x_q^{(k)}) - 
		\tnabla F_{t_{q,k}}(x_q^{(k)}) \|^2 \\
		& \qquad + 2 \langle \nabla \bar{F}_{q,k-1}(x_q^{(k)}) - \nabla 
		F_{t_{q,k}}(x_q^{(k)}),\nabla F_{t_{q,k}}(x_q^{(k)}) - \tnabla 
		F_{t_{q,k}}(x_q^{(k)}) \rangle|\mathcal{F}_{q,k-1}]].
		\end{split}    
		\end{equation}
		
		By \cref{assump_on_f}, and $F_t$ is $L_1$-Lipschitz implies that 
		$\sup_{x \in \constraint} \|\nabla F_t(x) \| \leq L_1$, we have 
		\begin{equation}
		\label{eq:decomp_12}
		\begin{split}
		\expect[\expect[\| \nabla \bar{F}_{q,k-1}(x_q^{(k)}) - \nabla 
		F_{t_{q,k}}(x_q^{(k)}) \|^2 |\mathcal{F}_{q,k-1}]] =& \expect 
		[\text{Var}(\nabla F_{t_{q,k}}(x_q^{(k)})|\mathcal{F}_{q,k-1})] \\
		\leq& \expect [\|\nabla F_{t_{q,k}}(x_q^{(k)}) \|^2] \\
		\leq& L_1^2.
		\end{split}    
		\end{equation}
		
		By \cref{assump_on_estimate}, we have
		\begin{equation}
		\label{eq:decomp_13}
		\begin{split}
		\expect[\expect[ \|\nabla F_{t_{q,k}}(x_q^{(k)}) - \tnabla 
		F_{t_{q,k}}(x_q^{(k)}) \|^2 |\mathcal{F}_{q,k-1}]] =& \expect[ \|\nabla 
		F_{t_{q,k}}(x_q^{(k)}) - \tnabla F_{t_{q,k}}(x_q^{(k)}) \|^2 ] \\
		=& \expect[\expect[ \|\nabla F_{t_{q,k}}(x_q^{(k)}) - \tnabla 
		F_{t_{q,k}}(x_q^{(k)}) \|^2 |\mathcal{F}_{q,k}]] \\
		\leq& \sigma_0^2.    
		\end{split}
		\end{equation}
		
		Moreover, we have 
		\begin{equation}
		\label{eq:decomp_14}
		\begin{split}
		&\expect[\expect[\langle \nabla \bar{F}_{q,k-1}(x_q^{(k)}) - \nabla 
		F_{t_{q,k}}(x_q^{(k)}),\nabla F_{t_{q,k}}(x_q^{(k)}) - \tnabla 
		F_{t_{q,k}}(x_q^{(k)}) \rangle|\mathcal{F}_{q,k-1}]] \\
		=& \expect[\langle \nabla \bar{F}_{q,k-1}(x_q^{(k)}) - \nabla 
		F_{t_{q,k}}(x_q^{(k)}),\nabla F_{t_{q,k}}(x_q^{(k)}) - \tnabla 
		F_{t_{q,k}}(x_q^{(k)}) \rangle] \\
		=& \expect[\expect[\langle \nabla \bar{F}_{q,k-1}(x_q^{(k)}) - \nabla 
		F_{t_{q,k}}(x_q^{(k)}),\nabla F_{t_{q,k}}(x_q^{(k)}) - \tnabla 
		F_{t_{q,k}}(x_q^{(k)}) \rangle|\mathcal{F}_{q,k}]]\\
		=& \expect[\langle \nabla \bar{F}_{q,k-1}(x_q^{(k)}) - \nabla 
		F_{t_{q,k}}(x_q^{(k)}),\expect[\nabla F_{t_{q,k}}(x_q^{(k)}) - \tnabla 
		F_{t_{q,k}}(x_q^{(k)})|\mathcal{F}_{q,k}] \rangle]\\
		=&0
		\end{split}
		\end{equation}
		where the last equation holds because $\tnabla F_{t}$ is an unbiased 
		estimator of $\nabla F_{t}$ for all $t$.
		
		By \cref{eq:decomp_11,eq:decomp_12,eq:decomp_13,eq:decomp_14}, we have
		\begin{equation}
		\label{eq:decomp_21}
		\expect [\| \nabla \bar{F}_{q,k-1}(x_q^{(k)}) - \tnabla 
		F_{t_{q,k}}(x_q^{(k)}) \|^2] \leq L_1^2 + \sigma_0^2 \triangleq 
		\sigma^2.
		\end{equation}
		
		Similarly, by law of iterated expectations and the unbiasedness of 
		$\tnabla F_t$, we have 
		\begin{equation}
		\label{eq:decomp_22}
		\begin{split}
		&\expect[\langle \nabla \bar{F}_{q,k-1}(x_q^{(k)}) - \tnabla 
		F_{t_{q,k}}(x_q^{(k)}), \nabla \bar{F}_{q,k-1}(x_q^{(k)}) - \nabla 
		\bar{F}_{q,k-2}(x_q^{(k-1)}) \rangle] \\
		=& \expect[ \expect[\langle \nabla \bar{F}_{q,k-1}(x_q^{(k)}) - \tnabla 
		F_{t_{q,k}}(x_q^{(k)}), \nabla \bar{F}_{q,k-1}(x_q^{(k)}) - \nabla 
		\bar{F}_{q,k-2}(x_q^{(k-1)}) \rangle | \mathcal{F}_{q,k-1}]] \\
		=& \expect[ \langle \expect[\nabla \bar{F}_{q,k-1}(x_q^{(k)}) - \tnabla 
		F_{t_{q,k}}(x_q^{(k)})| \mathcal{F}_{q,k-1}], \nabla 
		\bar{F}_{q,k-1}(x_q^{(k)}) - \nabla \bar{F}_{q,k-2}(x_q^{(k-1)}) 
		\rangle ] \\
		=& 0    
		\end{split}    
		\end{equation}
		and 
		\begin{equation}
		\label{eq:decomp_23}
		\begin{split}
		&\expect[\langle \nabla \bar{F}_{q,k-1}(x_q^{(k)}) - \tnabla 
		F_{t_{q,k}}(x_q^{(k)}), \nabla \bar{F}_{q,k-2}(x_q^{(k-1)}) - 
		d_q^{(k-1)} \rangle] \\
		=& \expect[\expect[\langle \nabla \bar{F}_{q,k-1}(x_q^{(k)}) - \tnabla 
		F_{t_{q,k}}(x_q^{(k)}), \nabla \bar{F}_{q,k-2}(x_q^{(k-1)}) - 
		d_q^{(k-1)} \rangle | \mathcal{F}_{q,k-1}, d_q^{(k-1)}  ]] \\
		=& \expect[\langle \expect[\nabla \bar{F}_{q,k-1}(x_q^{(k)}) - \tnabla 
		F_{t_{q,k}}(x_q^{(k)})|\mathcal{F}_{q,k-1}, d_q^{(k-1)}], \nabla 
		\bar{F}_{q,k-2}(x_q^{(k-1)}) - d_q^{(k-1)} \rangle ]\\
		=& 0.    
		\end{split}    
		\end{equation}
		
		Also, by Young's Inequality, we have 
		\begin{equation}
		\label{eq:decomp_24}
		\begin{split}
		\langle \nabla \bar{F}_{q,k-1}(x_q^{(k)}) - &\nabla 
		\bar{F}_{q,k-2}(x_q^{(k-1)}) , \nabla \bar{F}_{q,k-2}(x_q^{(k-1)}) - 
		d_q^{(k-1)} \rangle \\
		&\leq \frac{1}{2\alpha_k} \|\nabla \bar{F}_{q,k-1}(x_q^{(k)}) - \nabla 
		\bar{F}_{q,k-2}(x_q^{(k-1)}) \|^2 + 
		\frac{\alpha_k}{2}\Delta_q^{(k-1)}.  
		\end{split}
		\end{equation}
		
		Now we turn to bound $\| \nabla \bar{F}_{q,k-1}(x_q^{(k)}) - \nabla 
		\bar{F}_{q,k-2}(x_q^{(k-1)})\|^2 \triangleq z_{q,k}^2.$ In fact, we have
		
		\begin{equation*}
		\begin{split}
		\expect[z_{q,k}^2] &= \expect[\expect[\| \nabla 
		\bar{F}_{q,k-1}(x_q^{(k)}) - \nabla \bar{F}_{q,k-2}(x_q^{(k-1)})\|^2 
		|\mathcal{F}_{q,k-2}]]  \\
		&= \expect[\expect[\| \frac{\sum_{i=k}^K \nabla 
		F_{t_{q,i}}(x_q^{(k)})}{K-k+1} - \frac{\sum_{i=k-1}^K \nabla 
		F_{t_{q,i}}(x_q^{(k-1)})}{K-k+2}\|^2 |\mathcal{F}_{q,k-2}]] \\
		&= \expect[ \expect[ \| \frac{\sum_{i=k}^K \nabla 
		F_{t_{q,i}}(x_q^{(k)})-\nabla F_{t_{q,i}}(x_q^{(k-1)})}{K-k+2} + 
		\frac{\sum_{i=k}^K \nabla F_{t_{q,i}}(x_q^{(k)})}{(K-k+1)(K-k+2)} \\
		&\qquad - \frac{\nabla F_{t_{q,k-1}}(x_q^{(k-1)})}{K-k+2} \|^2 
		|\mathcal{F}_{q,k-2}]] \\
		&\leq \expect[ \expect[ (\sum_{i=k}^K \| \frac{\nabla 
		F_{t_{q,i}}(x_q^{(k)})-\nabla F_{t_{q,i}}(x_q^{(k-1)})}{K-k+2}\| + 
		\sum_{i=k}^K \| \frac{\nabla F_{t_{q,i}}(x_q^{(k)})}{(K-k+1)(K-k+2)} \| 
		\\
		& \qquad + \| \frac{\nabla F_{t_{q,k-1}}(x_q^{(k-1)})}{K-k+2} \|)^2 
		|\mathcal{F}_{q,k-2}]].
		\end{split}
		\end{equation*}
		where the inequality comes from the Triangle Inequality of norms.
		
		Recall the update rule where $x_q^{(k)} = x_q^{(k-1)} + 
		\frac{1}{K}v_q^{(k-1)}$ and the assumption that $F_t$ is $L_2$-smooth, 
		we have 
		\begin{equation*}
		\|\nabla F_{t_{q,i}}(x_q^{(k)})-\nabla F_{t_{q,i}}(x_q^{(k-1)}) \| \leq 
		L_2\frac{\| v_q^{(k)}\|}{K} = \frac{L_2R}{K}.    
		\end{equation*}
		
		Also by \cref{assump_on_f}, $\| \nabla F_{t_{q,i}}(x_q^{(k-1)}) \| \leq 
		L_1$. Therefore, we have
		\begin{equation}
		\label{eq:decomp_25}
		\begin{split}
		\expect[z_{q,k}^2] & \leq  [ (K-k+1) \frac{L_2R}{K}\frac{1}{K-k+2} + 
		\frac{L_1}{K-k+2} +(K-k+1)\frac{L_1}{(K-k+1)(K-k+2)}  ]^2\\
		&\leq \left( \frac{L_2R+2L_1}{K-k+2} \right)^2 \\
		&\triangleq \frac{G}{(K-k+2)^2}.
		\end{split}
		\end{equation}
		
		Combining 
		\cref{eq:delta,eq:decomp_21,eq:decomp_22,,eq:decomp_23,,eq:decomp_24,,eq:decomp_25},
		we have
		\begin{equation*}
		\begin{split}
		\expect[\Delta_q^{(k)}] &\leq \rho_k^2 \sigma^2 + (1-\rho_k)^2 
		\expect[\Delta_q^{(k-1)}] + (1-\rho_k)^2 \frac{G}{(K-k+2)^2} \\
		& \qquad + (1-\rho_k)^2\left[ \frac{G}{\alpha_k (K-k+2)^2} + 
		\alpha_k\expect[\Delta _q^{(k-1)}] \right].
		\end{split}    
		\end{equation*}
	\end{proof}
	
	Applying \cref{lem:delta_decomp} and setting $\alpha_k = \frac{\rho_k}{2}, 
	\forall k \in {1, 2, \cdots, K}$, we have
	\begin{equation*}
	\begin{split}
	\expect[\Delta_q^{(k)}] \leq \rho_k^2 \sigma^2 + 
	\frac{G}{(K-k+2)^2}(1-\rho_k)^2 \left( 1+ \frac{2}{\rho_k} \right) + 
	\expect[\Delta_q^{(k-1)}](1-\rho_k)^2\left(1+\frac{\rho_k}{2}\right). 
	\end{split}    
	\end{equation*}
	
	Note that if $0 < \rho_k \leq 1$, then we have 
	$$(1-\rho_k)^2\left(1+\frac{2}{\rho_k} \right) \leq 
	\left(1+\frac{2}{\rho_k} \right)$$
	and
	$$(1-\rho_k)^2\left(1+\frac{\rho_k}{2} \right) \leq (1 - \rho_k).$$
	
	So in this case, we have 
	\begin{equation}
	\label{eq:delta_iter}
	\expect[\Delta_q^{(k)}] \leq \rho_k^2 \sigma^2 + \frac{G}{(K-k+2)^2}\left( 
	1+ \frac{2}{\rho_k} \right) + \expect[\Delta_q^{(k-1)}](1-\rho_k).
	\end{equation}
	
	\begin{lemma}
		\label{lem:delta_final}
		Under the setting of \cref{thm:one_shot}, 
		we have
		\begin{equation*}
		\expect[\Delta_q^{(k)}] \leq
		\begin{cases}
		\frac{N}{(k+4)^{2/3}}, & \quad \text{when } 1 \leq k \leq \frac{K}{2}.\\
		\frac{N}{(K-k+1)^{2/3}},  & \quad \text{when } \frac{K}{2}+1 \leq k 
		\leq K.
		\end{cases}
		\end{equation*}
		where $N = \max\{ 5^{2/3}(L_1+M_0)^2, 4\sigma^2+32G,2.25\sigma^2+7G/3 
		\}$.
	\end{lemma}
	
	\begin{proof}[Proof of \cref{lem:delta_final}]
		When $1 \leq k \leq \frac{K}{2} + 1$, since $\rho_k = 
		\frac{2}{(k+3)^{2/3}}$, 
		we have $0 < \rho_k \leq 1$, and by \cref{eq:delta_iter}
		\begin{equation*}
		\label{eq:delta_iteration1}
		\begin{split}
		\expect[\Delta_q^{(k)}] &\leq \frac{4\sigma^2}{(k+3)^{4/3}} + 
		\frac{G}{k^2}[1+(k+3)^{2/3}] + \expect[\Delta_q^{(k-1)}]\left( 1 - 
		\frac{2}{(k+3)^{2/3}} \right)  \\
		&= \frac{4\sigma^2}{(k+3)^{4/3}} + 
		\frac{G}{(k+3)^2}\left(\frac{k+3}{k}\right)^2[1+(k+3)^{2/3}] + 
		\expect[\Delta_q^{(k-1)}]\left( 1 - \frac{2}{(k+3)^{2/3}} \right)  \\
		&\leq \frac{4\sigma^2}{(k+3)^{4/3}} + 
		\frac{G(1+3)^2}{(k+3)^2}[1+(k+3)^{2/3}] + 
		\expect[\Delta_q^{(k-1)}]\left( 1 - \frac{2}{(k+3)^{2/3}} \right) \\
		&\leq \frac{4\sigma^2}{(k+3)^{4/3}} + \frac{16G}{(k+3)^{4/3}} + 
		\frac{16G}{(k+3)^{4/3}} + \expect[\Delta_q^{(k-1)}]\left( 1 - 
		\frac{2}{(k+3)^{2/3}} \right) \\
		&= \frac{4\sigma^2+32G}{(k+3)^{4/3}} + \expect[\Delta_q^{(k-1)}]\left( 
		1 - \frac{2}{(k+3)^{2/3}} \right) \\
		&\triangleq \frac{N_0}{(k+3)^{4/3}} + \expect[\Delta_q^{(k-1)}]\left( 1 
		- \frac{2}{(k+3)^{2/3}} \right).
		\end{split}
		\end{equation*}
		
		Recall that $\Delta_q^{(k)} = \|\nabla \bar{F}_{q,k-1}(x_q^{(k)}) - 
		d_q^{(k)} \|^2$, and thus 
		\begin{equation*}
		\begin{split}
		\Delta_q^{(1)} &= \|\nabla \bar{F}_{q,0}(0) - d_q^{(1)}  \|^2 \\
		&= \| \frac{\sum_{i=1}^K \nabla F_{t_{q,i}}(0)}{K} - 
		\frac{2}{(1+3)^{2/3}} \tnabla F_{q,1}(0) \|^2 \\
		&\leq \left( \sum_{i=1}^{K} \|  \frac{\nabla F_{t_{q,i}}(0)}{K}\| + \| 
		\frac{2}{4^{2/3}} \tnabla F_{q,1}(0) 
		\| \right)^2 \\
		&\leq \left(K\frac{L_1}{K} + M_0 \right)^2 \\
		&= (L_1+M_0)^2. 
		\end{split}
		\end{equation*}
		
		Set $N_1 = \max \{ 5^{2/3}(L_1+M_0)^2, N_0 \}$, then we claim that 
		$\expect[\Delta_q^{(k)}] \leq \frac{N_1}{(k+4)^{2/3}}$ for any $k$ 
		satisfying $1 \leq k \leq \frac{K}{2} + 1$. We prove it by induction. 
		It holds for $k = 1$ because of the definition of $N_1$. Assume it 
		holds for $k - 1$, \emph{i.e.}, $\expect[\Delta_q^{(k-1)}] \leq 
		\frac{N_1}{(k+3)^{2/3}}$, then
		\begin{equation*}
		\begin{split}
		\expect[\Delta_q^{(k)}] &\leq \frac{N_1}{(k+3)^{4/3}} + 
		\expect[\Delta_q^{(k-1)}]\left( 1 - \frac{2}{(k+3)^{2/3}} \right) \\
		&\leq \frac{N_1}{(k+3)^{4/3}} + \frac{N_1}{(k+3)^{2/3}}\left( 1 - 
		\frac{2}{(k+3)^{2/3}} \right) \\
		&= \frac{N_1[(k+3)^{2/3} - 1]}{(k+3)^{4/3}}. 
		\end{split}    
		\end{equation*}
		
		Since $(k+4)^2 = k^2 + 8k+ 16 \leq k^2 + 6k +9 + 1 + 3(k+3) \leq k^2 + 
		6k +9 + 1 + 3(k+3)^{4/3} + 3(k+3)^{2/3} = [(k+3)^{2/3}+1]^3$, by taking 
		the cube roots of both sides, we have $(k+4)^{2/3} \leq (k+3)^{2/3}+1$,
		which implies that $[(k+3)^{2/3}-1](k+4)^{2/3} \leq [(k+3)^{2/3}-1] 
		[(k+3)^{2/3}+1] \leq (k+3)^{4/3}$, \emph{i.e.}, 
		$\frac{(k+3)^{2/3}-1}{(k+3)^{4/3}} \leq \frac{1}{(k+4)^{2/3}}$. So we 
		have $\expect[\Delta_q^{(k)}] \leq \frac{N_1}{(k+4)^{2/3}}$. By 
		induction, we have 
		\begin{equation}
		\label{eq:delta_bound1}
		\expect[\Delta_q^{(k)}] \leq \frac{N_1}{(k+4)^{2/3}}, \forall k \in 
		[\frac{K}{2}+1].    
		\end{equation}
		
		Now we turn to consider the case where $\frac{K}{2}+2 \leq k \leq K$. 
		Here we set $\rho_k = \frac{1.5}{(K-k+2)^{2/3}}$, note that $0 < \rho_k 
		\leq \frac{1.5}{2^{2/3}} < 1$, then we have
		\begin{equation*}
		\label{eq:delta_iteration2}
		\begin{split}
		\expect[\Delta_q^{(k)}] &\leq \frac{2.25\sigma^2}{(K-k+2)^{4/3}} + 
		\frac{G}{(K-k+2)^{2}} \left[1+\frac{4}{3}(K-k+2)^{2/3} \right] \\
		& \qquad + \expect[\Delta_q^{(k-1)}]\left[ 1- 
		\frac{1.5}{(K-k+2)^{2/3}}\right] \\
		&\leq \frac{2.25\sigma^2}{(K-k+2)^{4/3}} + \frac{G}{(K-k+2)^{4/3}} + 
		\frac{4}{3}\frac{G}{(K-k+2)^{4/3}}\\
		& \qquad + \expect[\Delta_q^{(k-1)}]\left[ 1- 
		\frac{1.5}{(K-k+2)^{2/3}}\right]\\
		&= \frac{2.25\sigma^2+7G/3}{(K-k+2)^{4/3}}+ 
		\expect[\Delta_q^{(k-1)}]\left[ 1- \frac{1.5}{(K-k+2)^{2/3}}\right] \\
		&\triangleq \frac{N_2}{(K-k+2)^{4/3}}+ \expect[\Delta_q^{(k-1)}]\left[ 
		1- \frac{1.5}{(K-k+2)^{2/3}}\right].
		\end{split}    
		\end{equation*}
		
		Define $N = \max\{N_1, N_2\}$, then we claim that 
		$\expect[\Delta_q^{(k)}] \leq 
		\frac{N}{(K-k+1)^{2/3}}$, for any $k$ satisfying $\frac{K}{2}+1 \leq k 
		\leq K$, 
		we will prove it by induction. When $k = \frac{K}{2}+1$, by 
		\cref{eq:delta_bound1}, we have
		$$\expect[\Delta_q^{(K/2+1)}] \leq \frac{N_1}{(K/2+1+4)^{2/3}} \leq 
		\frac{N}{(K/2)^{2/3}} = \frac{N}{(K-(K/2+1)+1)^{2/3}}.$$
		
		When it holds for $k-1$, \emph{i.e.}, $\expect[\Delta_q^{(k-1)}] \leq 
		\frac{N}{(K-k+2)^{2/3}}$, we have
		\begin{equation*}
		\begin{split}
		\expect[\Delta_q^{(k)}] &\leq \frac{N}{(K-k+2)^{4/3}}+ 
		\frac{N}{(K-k+2)^{2/3}}\frac{(K-k+2)^{2/3}-1.5}{(K-k+2)^{2/3}} \\
		&= \frac{N[(K-k+2)^{2/3}-0.5]}{(K-k+2)^{4/3}}. 
		\end{split}    
		\end{equation*}
		
		Since $[(K-k+2)^{2/3}-0.5](K-k+1)^{2/3} \leq 
		[(K-k+2)^{2/3}-0.5][(K-k+2)^{2/3}+0.5] \leq (K-k+2)^{4/3}$, 
		\emph{i.e.}, $\frac{(K-k+2)^{2/3}-0.5}{(K-k+2)^{4/3}} \leq 
		\frac{1}{(K-k+1)^{2/3}}$, so we have $\expect[\Delta_q^{(k)}] \leq 
		\frac{N}{(K-k+1)^{2/3}}$. By induction, we have
		\begin{equation*}
		\expect[\Delta_q^{(k)}] 
		\leq \frac{N}{(K-k+1)^{2/3}}, \forall k \in \{K/2+1, K/2+2, \cdots, K\}.
		\end{equation*}
		
		Since $N_1 \leq N$, by \cref{eq:delta_bound1}, we also have
		\begin{equation*}
		\expect[\Delta_q^{(k)}] \leq \frac{N}{(k+4)^{2/3}}, \forall k \in 
		[\frac{K}{2}+1].    
		\end{equation*}
	\end{proof}
	
	
	Recall that in \cref{eq:regret}, we have
	\begin{equation*}
	\expect[\mathcal{R}_T] \leq \sum_{q=1}^Q\sum_{k=1}^K \frac{1}{2\beta^{(k)}} 
	\expect[\Delta_q^{(k)}] + \frac{D^2}{2}Q\sum_{k=1}^K \beta^{(k)} + 
	K\mathcal{R}_Q^{\mathcal{E}} + \frac{L_2D^2}{2}Q.
	\end{equation*}
	
	So if we set 
	\begin{equation*}
	\label{eq:beta}
	\beta^{(k)} = 
	\begin{cases}
	(k+4)^{-1/3},& \quad \text{when } 1 \leq k \leq \frac{K}{2};\\
	(K-k+1)^{-1/3},& \quad \text{when } \frac{K}{2}+1 \leq k \leq K;
	\end{cases}
	\end{equation*}
	then by \cref{lem:delta_final}, we have
	\begin{equation*}
	\sum_{k=1}^{K/2}\frac{\expect[\Delta_q^{(k)}]}{\beta^{(k)}} \leq 
	\sum_{k=1}^{K/2} \frac{N}{(k+4)^{1/3}} 
	\leq \sum_{k=1}^{K/2}\frac{N}{k^{1/3}}  \leq 
	\int_0^{K/2}\frac{N}{x^{1/3}}\mathrm{d}x 
	=\frac{3N}{2}\left(\frac{K}{2}\right)^{2/3} \leq N K^{2/3},    
	\end{equation*}
	and
	$$\sum_{k=K/2+1}^{K}\frac{\expect[\Delta_q^{(k)}]}{\beta^{(k)}} \leq 
	\sum_{k=K/2+1}^{K}\frac{N}{(K-k+1)^{1/3}} = 
	\sum_{i=1}^{K/2}\frac{N}{i^{1/3}} \leq N K^{2/3}.$$
	
	Similarly, we have 
	$$\sum_{k=1}^{K/2}\beta^{(k)} = \sum_{k=1}^{K/2} \frac{1}{(k+4)^{1/3}} \leq 
	K^{2/3}$$
	and
	$$\sum_{k=K/2+1}^{K}\beta^{(k)} = \sum_{k=K/2+1}^{K}\frac{1}{(K-k+1)^{1/3}} 
	\leq K^{2/3}.$$
	
	Therefore, we have
	\begin{equation*}
	\begin{split}
	\expect[\mathcal{R}_T] &\leq \sum_{q=1}^Q NK^{2/3} + \frac{D^2}{2}Q\cdot 
	2K^{2/3} + K\mathcal{R}_Q^{\mathcal{E}} + \frac{L_2D^2}{2}Q \\
	&= (N+D^2)QK^{2/3} + K\mathcal{R}_Q^{\mathcal{E}} + \frac{L_2D^2}{2}Q.
	\end{split}
	\end{equation*}
	
	Set $Q = T^{2/5}, K=T^{3/5}$, and recall that $\mathcal{R}_Q^{\mathcal{E}} 
	\leq C\sqrt{Q} = CT^{1/5}$, we have 
	\begin{equation*}
	\label{eq:regret_bound}
	\expect[\mathcal{R}_T] \leq (N+C+D^2)T^{4/5} + \frac{L_2D^2}{2}T^{2/5}.  
	\end{equation*}
\end{proof}

\section{Properties of Smoothed Functions}\label{app:smooth_property}
\begin{lemma}\label{lem:smooth_property}
	If $F$ is monotone, continuous DR-submodular, $L_1$-Lipschitz, and 
	$L_2$-smooth, then so is $\hat{F}_\delta$, and for all $x$ we have $| 
	\hat{F}_\delta(x) - F(x) | \leq L_1 \delta$.
\end{lemma}

\begin{proof}
	By Lemmas 1 and 2 of 
	\citep{chen2019black}, we conclude that $\hat{F}_\delta$ is also monotone 
	continuous DR-submodular, $L_1$-Lipschitz and it holds that
	\begin{equation*}
	|\hat{F}_\delta(x) - F(x)| \leq L_1 \delta.     
	\end{equation*}
	
	For any $x, y$ in the domain of $\hat{F}_\delta$, we have
	\begin{equation*}
	\begin{split}
	\|\nabla \hat{F}_\delta(x) - \nabla \hat{F}_\delta(y)\| &= \|\nabla 
	\expect[F(x+\delta v)] - \nabla \expect[F(y+\delta v)]\| \\
	&= \|\expect [\nabla F(x+\delta v)] - \expect[\nabla F(y+\delta v)]\| \\
	&= \|\expect [\nabla F(x+\delta v) - \nabla F(y+\delta v)]\|  \\
	&\leq \expect[\|\nabla F(x+\delta v) - \nabla F(y+\delta v)\|] \\
	&\leq \expect[L_2 \|x-y \|] \\
	&= L_2 \|x-y\|.
	\end{split}    
	\end{equation*}
	
	So $\hat{F}_\delta$ is also $L_2$-smooth.   
\end{proof}

\section{Construction of $\delta$-Interior}\label{app:shrink_down_closed}
\cref{fig:delta} is the illustrations of $\delta$-interior and the construction 
method as discussed in \cref{lem:shrink_down_closed}.
\begin{figure}[ht]
	\begin{subfigure}[b]{0.45\textwidth}
		\hbox{\hspace{7.5em}\includegraphics[height=0.3\textwidth]{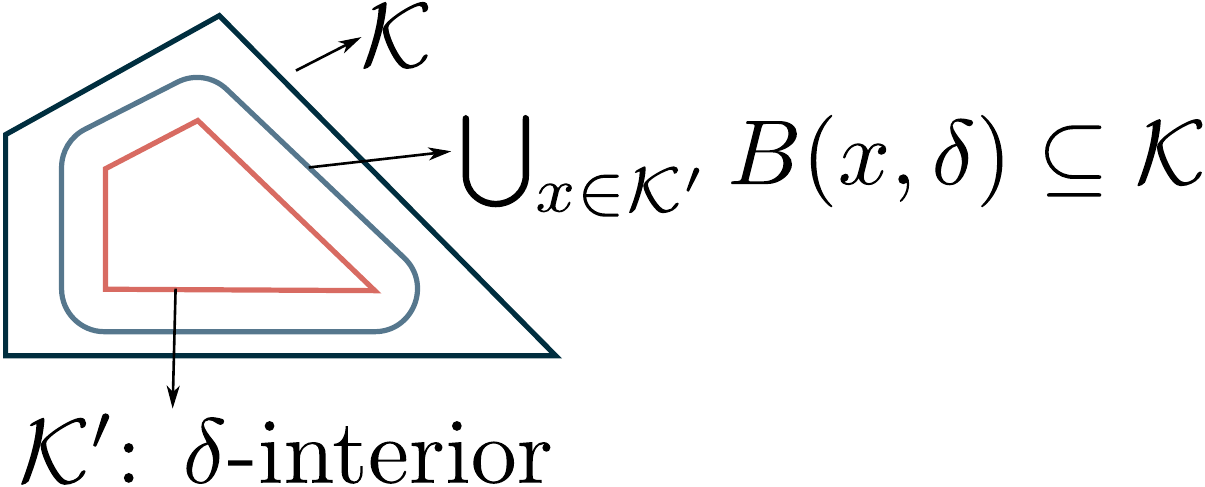}}
		\caption{Example of $\delta$-interior}
		\label{subfig:example}
	\end{subfigure}
	\hfill
	\begin{subfigure}[b]{0.45\textwidth}
		\hbox{\hspace{6em}\includegraphics[height=0.3\textwidth]{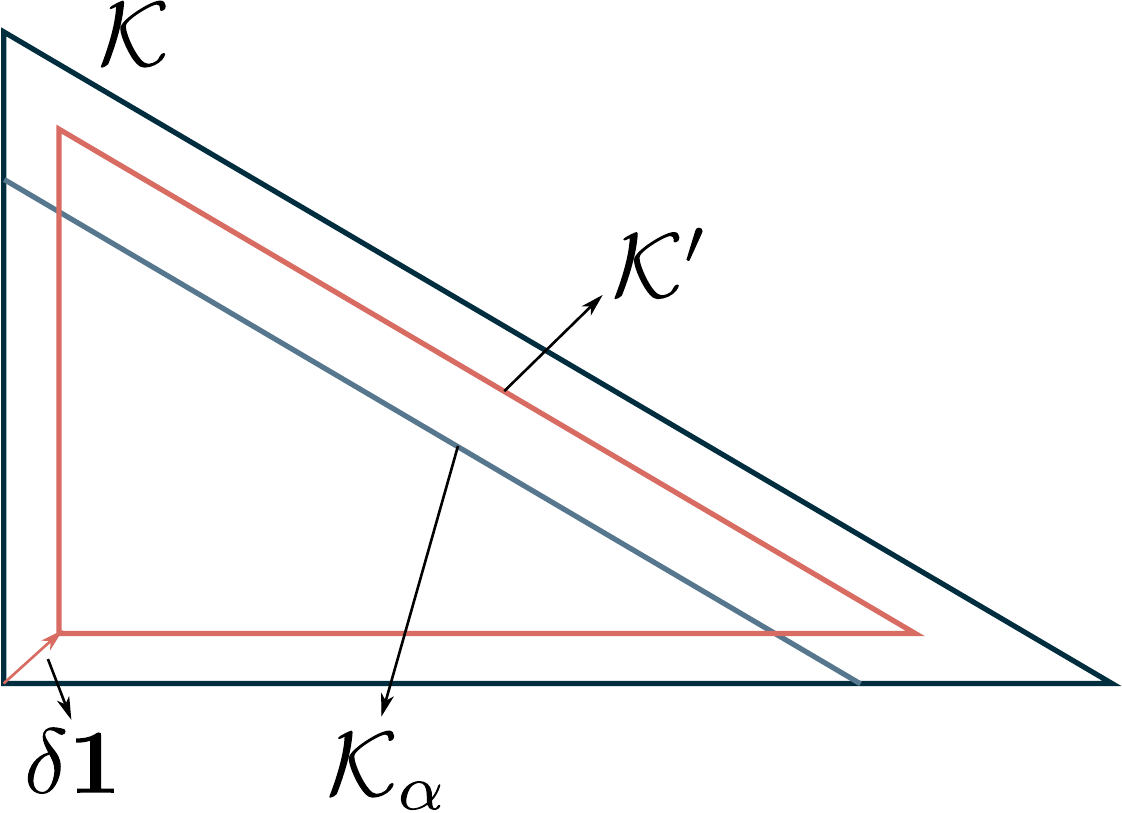}}
		\caption{Construction of $\delta$-interior}\label{subfig:construction}
	\end{subfigure}
	\caption{$\delta$-interior}
	\label{fig:delta}
\end{figure}

Now we turn to prove \cref{lem:shrink_down_closed}. We first show the following 
auxiliary lemma.

\begin{lemma}
	\label{lem:sphere}
	Consider a ball centered at the origin $o$. If point $a$ resides on the 
	sphere 
	but not in the non-negative orthant, there must exist a point $b$ on the 
	sphere  such that all the components 
	of $\overrightarrow{ab}$ are positive and all the components of 
	$\overrightarrow{ob}$ are non-negative. 
\end{lemma}

\begin{proof}[Proof of \cref{lem:sphere}]
	Without loss of generality, we assume the Cartesian coordinates of $a$ are 
	$(-\epsilon_1, -\epsilon_2, \cdots, -\epsilon_k, \epsilon_{k+1}, \cdots, 
	\epsilon_d)$, where $\epsilon_i >0, \forall i \in [k]$, $\epsilon_j \geq 0, 
	\forall j \in \{k+1, \cdots, d \}$, and $k \in [d]$. In order to find a 
	point $b$, we first define the symmetric point $b' = (\epsilon_1, 
	\epsilon_2, \cdots, \epsilon_k, \epsilon_{k+1}, \cdots, \epsilon_d)$.
	
	If $k=d$, we can set $b=b'$, then $b$ is on the sphere, 
	$b_i-a_i=2\epsilon_i >0$, and $b_i = \epsilon_i > 0, \forall i \in [d]$. 
	
	If $k < d$, we can add some perturbations on $b'$. Let $\epsilon = 
	\min\{\epsilon_1, \epsilon_2, \cdots, \epsilon_k\} > 0, A = 
	\frac{2\epsilon\sum_{i=1}^k \epsilon_i -k\epsilon^2}{d-k} > 0$, and set $b 
	= b' + (-\epsilon, -\epsilon,\cdots, -\epsilon, 
	\sqrt{A+\epsilon_{k+1}^2}-\epsilon_{k+1}, 
	\cdots,\sqrt{A+\epsilon_{d}^2}-\epsilon_{d}) = (\epsilon_1 -\epsilon, 
	\epsilon_2-\epsilon, \cdots, \epsilon_k - \epsilon, 
	\sqrt{A+\epsilon_{k+1}^2}, \cdots, \sqrt{A+\epsilon_{d}^2})$. Note that 
	$|ob|^2=\sum_{i=1}^k(\epsilon_i -\epsilon)^2 + \sum_{j=k+1}^d 
	(A+\epsilon_j^2) = \sum_{i=1}^k \epsilon_i^2 
	-2\epsilon\sum_{i=1}^k\epsilon_i +k \epsilon^2 + 2\epsilon\sum_{i=1}^k 
	\epsilon_i -k\epsilon^2 + \sum_{j=k+1}^d\epsilon_j^2 = \sum_{l=1}^d 
	\epsilon_l^2 = |oa|^2$, so $b$ is also on the sphere. Moreover, $b_i - a_i= 
	2\epsilon_i - \epsilon > 0, \forall i \in [k]$, $b_j - a_j = 
	\sqrt{A+\epsilon_j^2} - \epsilon_j >0, \forall j \in \{k+1, \cdots,d\}$, 
	and $b_i = \epsilon_i - \epsilon \geq 0, \forall i \in [k], b_j = 
	\sqrt{A+\epsilon_j^2} > 0, \forall j \in \{k+1, \cdots,d\}$. 
	
	Therefore, all the scalar components of $\overrightarrow{ab}$ are positive, 
	and all the scalar components of $\overrightarrow{ob}$ are non-negative.
\end{proof}

\begin{figure}[ht]
	\begin{subfigure}[b]{0.5\textwidth}
		\hbox{\hspace{5.7em}\includegraphics[height=0.3\textwidth]{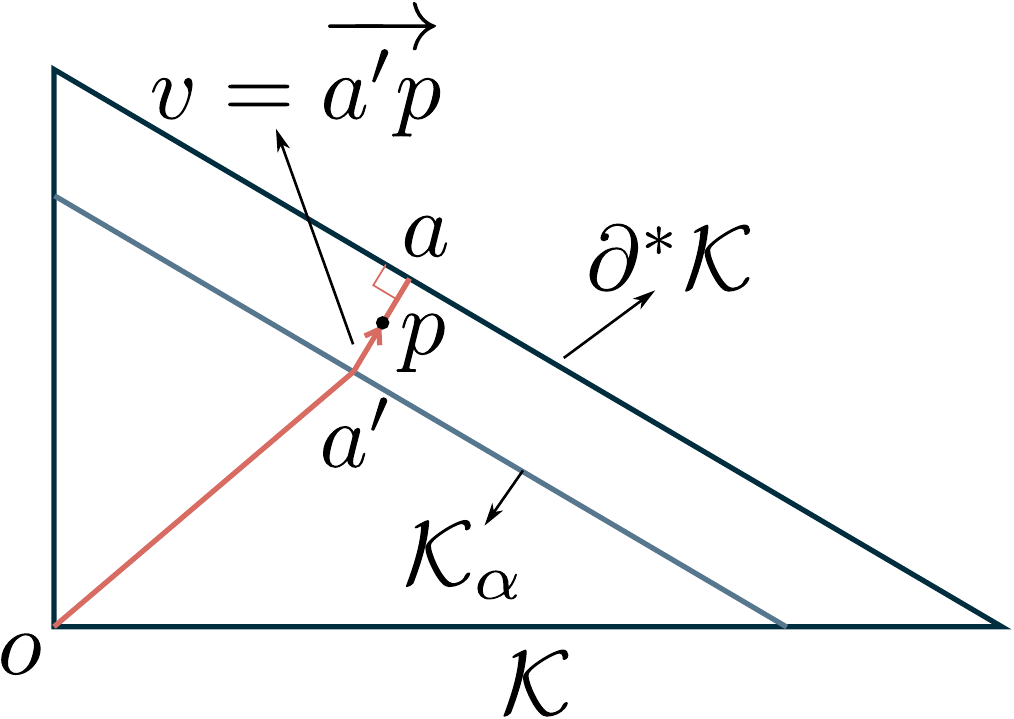}}
		\caption{}
		\label{subfig:illu_a}
	\end{subfigure}
	\hfill
	\begin{subfigure}[b]{0.5\textwidth}
		\hbox{\hspace{6.4em}\includegraphics[height=0.42\textwidth]{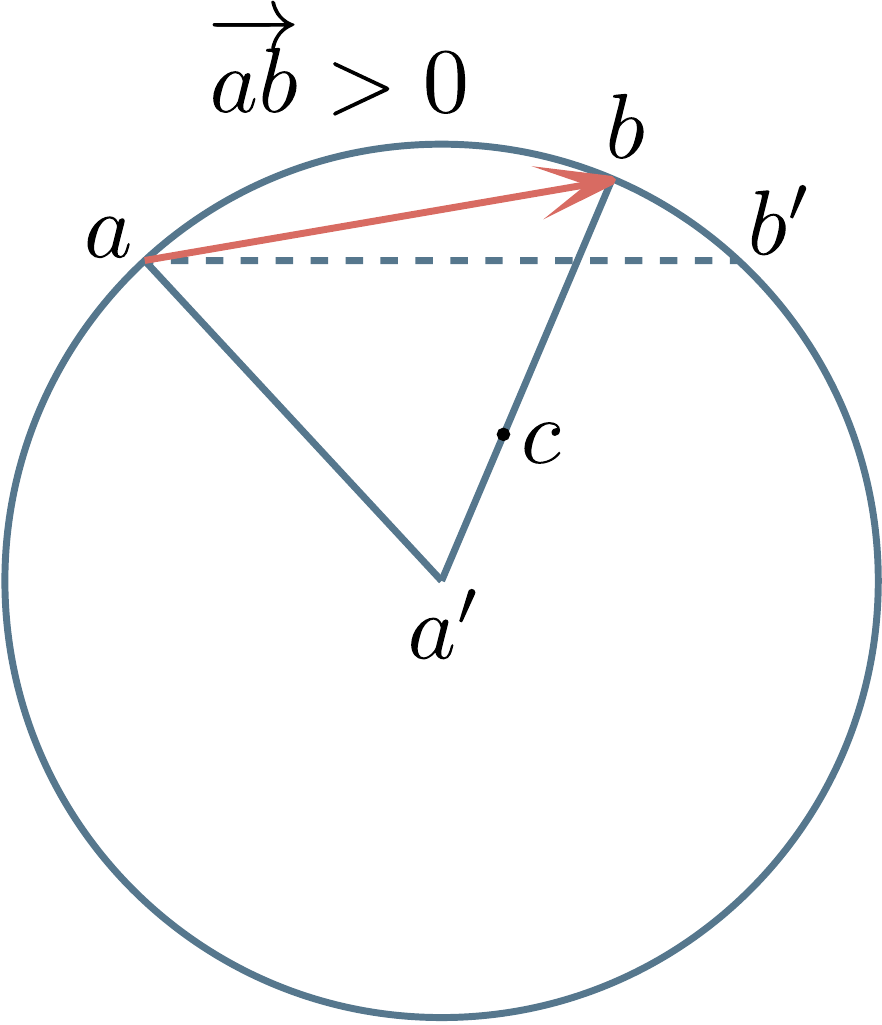}}
		\caption{}
		\label{subfig:illu_b}
	\end{subfigure}
	
	\begin{subfigure}[b]{0.5\textwidth}
		\hbox{\hspace{6.5em}\includegraphics[height=0.385\textwidth]{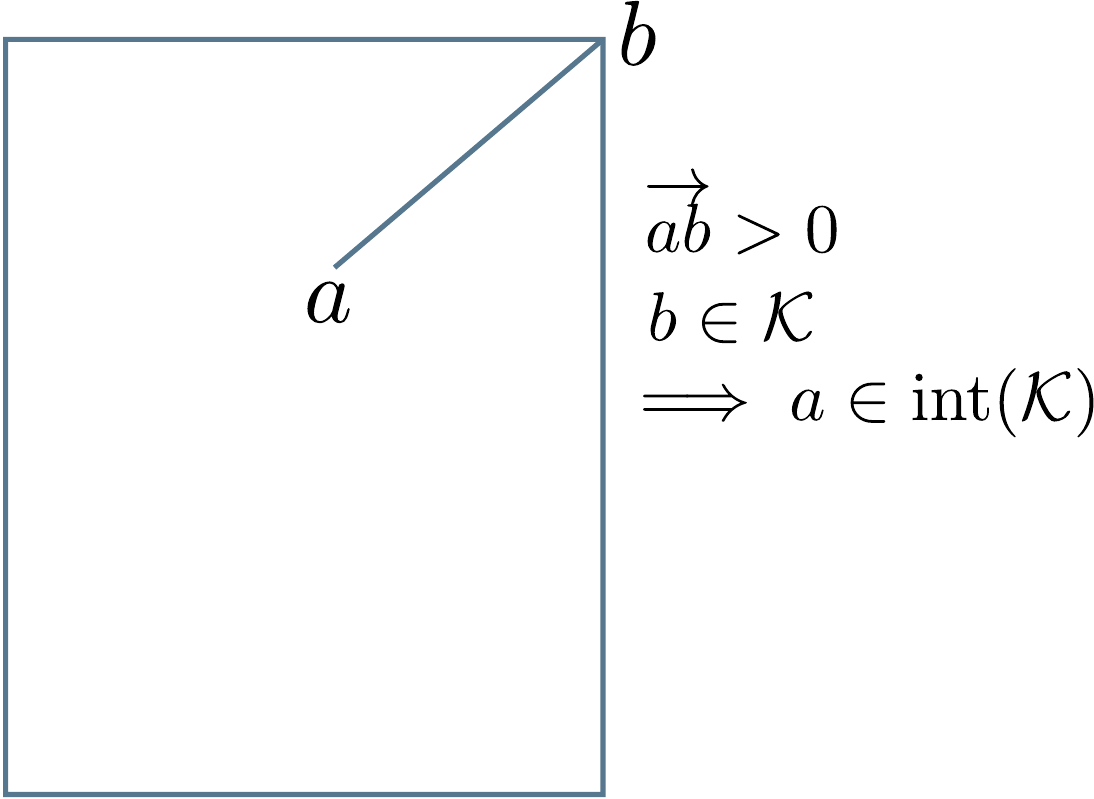}}
		\caption{}
		\label{subfig:illu_c}
	\end{subfigure}
	\hfill
	\begin{subfigure}[b]{0.5\textwidth}
		\hbox{\hspace{4.5em}\includegraphics[height=0.32\textwidth]{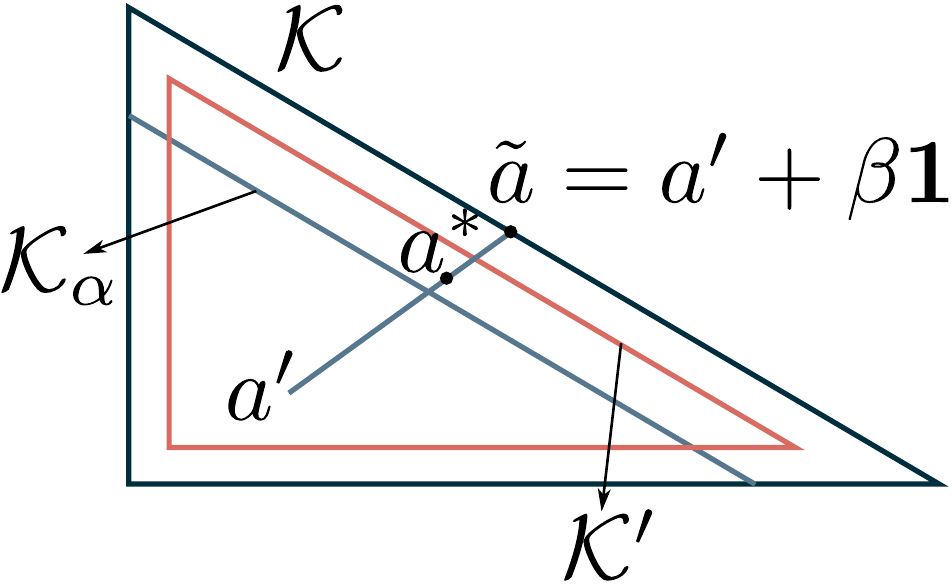}}
		\caption{}
		\label{subfig:illu_d}
	\end{subfigure}
	\label{fig:illu}
	\caption{Illustrations for Proof of \cref{lem:shrink_down_closed}}
\end{figure}

\begin{proof}[Proof of \cref{lem:shrink_down_closed}]
	Since $\constraint$ is convex, compact, and down-closed, and only shrinkage 
	and translation are involved, so $\constraint'$ is also convex, compact, 
	and down-closed. In order to prove that $\constraint'$ is a 
	$\delta$-interior of $\constraint$, note that thanks to the $\delta \one$ 
	translation, the distance between $\constraint'$ and the face which 
	contains 0 (\emph{i.e.}, the set $\partial^0 \constraint = \{ x \in 
	\partial \constraint | \exists i \in [d] \text{ such that } x_i =0 \}$), is 
	no less than $\delta$. In other words, for every $a^* \in \constraint'$, we 
	have $\inf_{x \in \partial^0 \constraint}d(x, a^*) \geq \delta$.
	
	So we only need to consider the remaining points on $\partial \constraint$, 
	which we denote as $\partial^* \constraint = \partial \constraint \setminus 
	\partial^0 \constraint = \{x \in \partial \constraint | \forall i \in [d], 
	x_i >0 \}$. We also denote the closure of $\partial^* \constraint$ as 
	$\text{cl}(\partial^* \constraint)$, which is a subset of $\partial 
	\constraint$. 
	Since for every point $a^* \in \constraint'$, there is a point $a' = a^* 
	-\delta \one \in \constraint_\alpha$, and $|a'a^*|=\sqrt{d}\delta$, we can 
	first analyze $\inf_{s \in \partial^* \constraint}d(s,a')$, and then upper 
	bound $\inf_{s \in \partial^* \constraint}d(s,a^*)$ by triangle inequality. 
	
	For any point $a' \in \constraint_\alpha$, suppose the point $a \in 
	\text{cl}(\partial^* \constraint)$ satisfies $|aa'| = \inf_{x \in 
	\partial^* \constraint} d(x, a')$ (\cref{subfig:illu_a}).
	We claim that all the scalar components of the vector 
	$\overrightarrow{a'a}$ are non-negative. We will prove it by contradiction. 
	Consider a ball with $a'$ as the center and $|a'a|$ as the radius.
	If we regard $a'$ as the origin $o$, then the assumption that 
	$\overrightarrow{a'a}$ has negative scalar component is equivalent to that 
	$a$ is not in the non-negative orthant.
	
	By \cref{lem:sphere}, there exists a point $b$, such that $|a'b| = |a'a|$, 
	all the scalar components of $\overrightarrow{ab}$ are positive, and all 
	the scalar components of $\overrightarrow{a'b}$ are non-negative 
	(\cref{subfig:illu_b}). Then we claim $b \in \constraint$, which will be 
	also proved by contradiction. If $b \notin \constraint$, since $a \in 
	\text{cl}(\partial^* \constraint)$ implies $a_i \geq 0, \forall i$, the 
	fact that all the scalar components of $\overrightarrow{ab}$ are positive 
	implies $b_i > 0, \forall i$. 
	
	Since $a' \in \constraint_\alpha$, there must be a point $c \neq a'$ in the 
	line segment $\overline{a'b}$ such that $c \in \partial \constraint$. 
	To prove it, note that $a' \in \constraint_\alpha \implies a' \in 
	(1-\alpha)\constraint$, and $(\sqrt{d}+1)\delta B^d_{\geq0} = \alpha 
	rB^d_{\geq0} \subseteq \alpha \constraint$. So, $a' + (\sqrt{d}+1)\delta 
	B^d_{\geq0} \subseteq (1-\alpha)\constraint + \alpha \constraint = 
	\constraint$ by the convexity of $\constraint$. On the other hand, since 
	all the scalar components of $\overrightarrow{a'b}$ are non-negative, the 
	intersection between the line segment $\overline{a'b}$ and the set $a' + 
	(\sqrt{d}+1)\delta B^d_{\geq0}$ must contains point other than $a'$. We 
	denote this point as $c'$, then $c' \in a' + (\sqrt{d}+1)\delta B^d_{\geq0} 
	\subseteq \constraint$. By the convexity of $\constraint$, the continuity 
	of the line segment $\overline{a'b}$, and the assumption that $b \notin 
	\constraint$, there must be a point $c \neq a'$ in $\overline{a'b}$ such 
	that $c \in \partial \constraint$.
	
	Then $c \neq a', a'_i \geq 0, b_i >0, c\in \overline{a'b}$ imply that $c_i 
	> 0, \forall i$, thus $c \in \partial^* \constraint$. Moreover, since we 
	assume $b \notin \constraint$, we have $|a'c| < |a'b| = |a'a|$, which is 
	contradictory with the assumption that $|a'a| = \inf_{x \in \partial^* 
	\constraint} d(x, a')$.
	
	So we must have $b \in \constraint$. Since the scalar components of 
	$\overrightarrow{ab}$ all all positive, and $\constraint$ is down-closed 
	($0 \leq x \leq y, y\in \constraint \implies x \in \constraint$), we 
	conclude that $a$ is an interior point of $\constraint$ 
	(\cref{subfig:illu_c}), which is contradictory to the assumption that $a 
	\in \text{cl}(\partial^* \constraint)$. So we have proved that all the 
	scalar components of the vector $\overrightarrow{a'a}$ are non-negative.
	
	Then we proceed to show $|a'a| \geq (\sqrt{d}+1)\delta$. Let $v$ be the 
	vector $\frac{(\sqrt{d}+1)\delta}{|a'a|} \overrightarrow{a'a}$, and $p$ be 
	the point such that $\overrightarrow{a'p} = v$ (\cref{subfig:illu_a}). Then 
	$|v| = (\sqrt{d}+1)\delta$ and all the scalar components of $v$ are 
	non-negative, \emph{i.e.}, $v \in (\sqrt{d}+1)\delta B^d_{\geq0} = \alpha 
	rB^d_{\geq0} \subseteq \alpha \constraint$. We also have $a' \in 
	\constraint_\alpha = (1-\alpha)\constraint$, thus $p \in 
	(1-\alpha)\constraint + \alpha \constraint = \constraint$ by the convexity 
	of $\constraint$. Since $a \in \text{cl}(\partial^* \constraint)$, we have 
	$|a'a| \geq |a'p| = |v| = (\sqrt{d}+1)\delta.$
	
	Let $a^* = a' + \delta \one$ be the translated point of $a'$. Then for any 
	point $s \in \partial^* \constraint$, by triangle inequality, we have 
	$|a^*s| \geq |a's| - |a'a^*| \geq |a'a| - |a'a^*| \geq (\sqrt{d}+1)\delta - 
	\sqrt{d}\delta = \delta$. So $\inf_{x \in \partial^* \constraint} d(x, a^*) 
	\geq \delta$. Since $a'$ can be arbitrary point in $\constraint_\alpha$, 
	the inequality holds for every point $a^* \in \constraint'$. Recall that we 
	have proved that for every $a^* \in \constraint', \inf_{x \in \partial^0 
	\constraint}d(x, a^*) \geq \delta$, where $\partial^0 \constraint = \{ x 
	\in \partial \constraint | \exists i \in [d] \text{ such that } x_i =0 \} = 
	\partial \constraint \setminus \partial^* \constraint$. Therefore, we 
	conclude that for every point $a^* \in \constraint', \inf_{x \in \partial 
	\constraint}d(x, a^*) \geq \delta$.
	
	So we only need to prove $\constraint' \subseteq \constraint$. For every 
	$a^* \in \constraint'$, since $a' = a^* -\delta \one \in 
	\constraint_\alpha$, there must be a positive $\beta$, such that $\tilde{a} 
	= a' + \beta \one \in \partial^* \constraint$ (\cref{subfig:illu_d}). We 
	have shown that $\inf_{x \in \partial^* \constraint}d(x,a') \geq 
	(\sqrt{d}+1)\delta$, so $\beta \geq \frac{\sqrt{d}+1}{\sqrt{d}}\delta > 
	\delta$. So $a^* = a' + \delta \one$ must be in the segment of 
	$\overline{a'\tilde{a}}$. Then we have $a^* \in \constraint$, by the fact 
	that $a', \tilde{a} \in \constraint$, and the convexity of $\constraint$. 
	Therefore, $\constraint' \subseteq \constraint$, and thus $\constraint'$ is 
	a $\delta$-interior of $\constraint$.
	
	Now we turn to analyze $d(\constraint,\constraint')$. For any point $x \in 
	\constraint$, we define $x' = (1-\alpha)x \in \constraint_\alpha$, and have 
	$|xx'| = \alpha |ox| \leq \alpha R$. Let $x^* = x' + \delta \one \in 
	\constraint'$, then $|xx^*| \leq |xx'| + |x'x^*| \leq \alpha R + \sqrt{d} 
	\delta = [\sqrt{d}(\frac{R}{r}+1) + \frac{R}{r}]\delta$. Thus 
	$d(\constraint,\constraint') \leq [\sqrt{d}(\frac{R}{r}+1) + 
	\frac{R}{r}]\delta$.
\end{proof}

\section{Analysis of \cref{alg:bandit}}
\subsection{General Constraint Set}\label{app:general_K}
We first state a necessary assumption on the $ \delta $-interior $\constraint'$.

\begin{assump}
	\label{assump_on_constraint'}
	For sufficiently small $\delta >0$, the $\delta$-interior $\constraint'$ is 
	convex and compact, and has lower bound $\underline{u}$ such that 
	$\forall x \in \constraint', x \geq \underline{u}$. 
	We also assume that the discrepancy 
	satisfies 
	$d(\constraint, \constraint') \leq c_1 \delta^\gamma$, where $c_1, 
	\gamma > 0$.
\end{assump}

Note that 
we have $\sup_{x,y \in \constraint'}\|x-y \| \leq D, \sup_{x \in 
	\constraint'}\|x-\lb \| \leq R$, 
where $D, R$ are the diameter and radius of $\constraint$. In other words, the 
bounds for $\constraint$ also hold for $\constraint'$. 

Also, if the constraint set $\constraint$ satisfies \cref{assump_on_K} and 
is down-closed, \cref{lem:shrink_down_closed} shows that one can construct a $ 
\delta $-interior $\constraint'$ that obeys \cref{assump_on_constraint'}.

Now with the assumption on the reward functions $F_t$ 
(\cref{assump_on_f,bandit_assump_on_f}), and those on $\constraint$ and 
$\constraint'$ 
(\cref{assump_on_K,assump_on_constraint'}), we show \cref{alg:bandit} achieves 
a sublinear $(1-1/e)$-regret bound of 
$O(T^{\frac{3+5\min\{1,\gamma\}}{3+6\min\{1,\gamma\}}})$. 



\begin{theorem}\label{thm:bandit}
	Under 
	\cref{assump_on_oracle,assump_on_f,assump_on_K,assump_on_constraint',bandit_assump_on_f},
	if we set $\delta = c_2T^{-\frac{1}{3+6\min\{1,\gamma\}}}, Q = 
	T^{\frac{2\min\{1,\gamma\}}{3+6\min\{1,\gamma\}}}, L = 
	T^{\frac{3+4\min\{1,\gamma\}}{3+6\min\{1,\gamma\}}}, K = 
	T^{\frac{1+\min\{1,\gamma\}}{1+2\min\{1,\gamma\}}}, \eta_k = \frac{1}{K}, 
	\rho_k=\frac{2}{(k+2)^{2/3}}$, where $c_2 >0$ is a constant such that 
	$\delta$ is sufficiently small as required by \cref{assump_on_constraint'},
	then the expected $(1-1/e)$-regret of \cref{alg:bandit} is at most
	\begin{equation*}
	\begin{split}
	\expect[\mathcal{R}_T] 
	\leq& \left[(1-1/e)c_1c_2^\gamma L_1 + (2-1/e)c_2L_1+ 
	2M_1+\frac{3\cdot4^{1/6}d^2M_1^2}{c_2}+\frac{3D^2}{4c_2}+C 
	\right]T^{\frac{3+5\min\{1,\gamma\}}{3+6\min\{1,\gamma\}}} \\
	& \quad + 
	\frac{3c_2[2L_1^2+(3L_2R+2L_1)^2]}{4^{1/3}}T^{\frac{1+5\min\{1,\gamma\}}{3+6\min\{1,\gamma\}}}
	 + \frac{L_2D^2}{2}T^{\frac{\min\{1,\gamma\}}{1+2\min\{1,\gamma\}}}.
	\end{split}    
	\end{equation*}
\end{theorem}

\begin{proof}[Proof of \cref{thm:bandit}]
	Since $x_q^{(1)} = \underline{u}$ and $\eta_k = 1/K$, $x_q^{(k)}$ is 
	actually a convex combination of $\underline{u}, v_q^{(1)}, v_q^{(2)}, 
	\cdots, v_q^{(k-1)}$. Then $\underline{u} \in \constraint', v_q^{(i)} \in 
	\constraint', \forall i \in [K]$ implies $x_q^{(k)} \in \constraint', 
	\forall k \in[K+1]$. So for $k \in [K], y_{t_{q,k}} = x_q^{(k)} + \delta 
	u_{q,k} \in \constraint$; for $t \in \{(q-1)L+1,\cdots, qL\} \setminus \{ 
	t_{q,1}, \cdots, t_{q,K} \} $, $y_t = x_q = x_q^{(K+1)} \in \constraint' 
	\subseteq \constraint$. In other words, all the points that we play fall on 
	the constraint set $\constraint$. 
	
	We also note that as discussed before, 
	the regret bound for online linear oracle, $\mathcal{R}_t^{\mathcal{E}} 
	\leq C\sqrt{t}$ can be achieved by algorithms such as Online Gradient 
	Descent. 
	
	Then we define
	\begin{equation*}
	\hat{F}_{t,\delta}(x) = \expect_{v \sim B^d}[F_t(x + \delta v)]  
	\end{equation*}
	as the $\delta$-smoothed version of $F_t$. We omit the $\delta$ in the 
	subscript for simplicity in the rest of the proof. Since $F_t$ is 
	$L_1$-Lipschitz, by \cref{lem:smooth_property} in 
	\cref{app:smooth_property}, we have
	\begin{equation*}
	| \hat{F}_t(x) - F_t(x) | \leq L_1\delta.  
	\end{equation*}
	
	Therefore, if we define $x^* = \argmax_{x \in \constraint}\sum_{t=1}^T 
	F_t(x), x^*_\delta = \argmax_{x \in \constraint'}\sum_{t=1}^T F_t(x)$, the 
	$(1-1/e)$-regret with horizon $T$ is
	\begin{equation*}
	\begin{split}
	\mathcal{R}_T &= \sum_{t=1}^T [(1-1/e)F_t(x^*) - F_t(y_t)] \\
	&= \sum_{t=1}^T [(1-1/e)F_t(x^*) - (1-1/e)F_t(x^*_\delta) + 
	(1-1/e)F_t(x^*_\delta) - F_t(y_t)] \\
	&= (1-1/e)\sum_{t=1}^T[F_t(x^*) - F_t(x^*_\delta)]+ 
	\sum_{t=1}^T[(1-1/e)\hat{F}_t(x^*_\delta)-\hat{F}_t(y_t)] \\
	&\quad + \sum_{t=1}^T(1-1/e)[F_t(x^*_\delta)-\hat{F}_t(x^*_\delta)] - 
	\sum_{t=1}^T[F_t(y_t)-\hat{F}_t(y_t)] \\
	& \leq (1-1/e)\sum_{t=1}^T[F_t(x^*) - F_t(x^*_\delta)]+ 
	\sum_{t=1}^T[(1-1/e)\hat{F}_t(x^*_\delta)-\hat{F}_t(y_t)] + 
	T(1-1/e)L_1\delta + TL_1\delta \\
	& = (1-1/e)\sum_{t=1}^T[F_t(x^*) - F_t(x^*_\delta)]+ 
	\sum_{t=1}^T[(1-1/e)\hat{F}_t(x^*_\delta)-\hat{F}_t(y_t)] + 
	(2-1/e)L_1T\delta.
	\end{split}]    
	\end{equation*}
	
	Suppose $x' \in \constraint'$ such that $\|x^* - x' \| = d(x^*,x') = d(x^*, 
	\constraint') \leq d(\constraint,\constraint') \leq c_1\delta^\gamma$, then 
	we have
	\begin{equation*}
	\begin{split}
	\sum_{t=1}^T[F_t(x^*) - F_t(x^*_\delta)] &=  \sum_{t=1}^T[F_t(x^*) - 
	F_t(x') + F_t(x') - F_t(x^*_\delta)] \\
	&= \sum_{t=1}^T[F_t(x^*) - F_t(x')] + [\sum_{t=1}^TF_t(x') - \sum_{t=1}^T 
	F_t(x^*_\delta)] \\
	&\leq \sum_{t=1}^T[L_1 \|x^* - x' \|] + 0 \\
	&\leq c_1L_1T\delta^\gamma,
	\end{split}    
	\end{equation*}
	where the first inequality holds thanks to the optimality of $x^*_\delta$ 
	and the assumption that $F_t$ is $L_1$-Lipschitz.
	
	Moreover, we have
	\begin{equation*}
	\begin{split}
	\hat{\mathcal{R}}_T &\triangleq 
	\sum_{t=1}^T[(1-1/e)\hat{F}_t(x^*_\delta)-\hat{F}_t(y_t)]   \\
	&= \sum_{q=1}^Q \sum_{i=1}^L 
	[(1-1/e)\hat{F}_{t_{q,i}}(x^*_\delta)-\hat{F}_{t_{q,i}}(x_q)] + 
	\sum_{q=1}^Q \sum_{k=1}^K 
	[\hat{F}_{t_{q,k}}(x_q)-\hat{F}_{t_{q,k}}(y_{t_{q,k}})] \\
	& \leq \sum_{q=1}^Q \sum_{i=1}^L 
	[(1-1/e)\hat{F}_{t_{q,i}}(x^*_\delta)-\hat{F}_{t_{q,i}}(x_q)] + 
	\sum_{q=1}^Q \sum_{k=1}^K[2M_1] \\
	& = \sum_{q=1}^Q \sum_{i=1}^L 
	[(1-1/e)\hat{F}_{t_{q,i}}(x^*_\delta)-\hat{F}_{t_{q,i}}(x_q)] + 2M_1QK
	\end{split}
	\end{equation*}
	where the inequality holds since 
	\begin{equation*}
	|\hat{F}_{q,t_k}(x)| = |\expect_{v \sim B^n}[F_{q,t_k}(x + \delta v)]| \leq 
	\expect[|F_{q,t_k}(x + \delta v)|] \leq M_1.    
	\end{equation*}
	
	So by now, we have
	\begin{equation*}
	\mathcal{R}_T \leq (1-1/e)c_1L_1T\delta^\gamma + (2-1/e)L_1T\delta + 2M_1QK 
	+  \sum_{q=1}^Q \sum_{i=1}^L 
	[(1-1/e)\hat{F}_{t_{q,i}}(x^*_\delta)-\hat{F}_{t_{q,i}}(x_q)].   
	\end{equation*}
	
	In order to upper bound $\sum_{q=1}^Q \sum_{i=1}^L 
	[(1-1/e)\hat{F}_{t_{q,i}}(x^*_\delta)-\hat{F}_{t_{q,i}}(x_q)]$, we first 
	define the average function:
	\begin{equation*}
	\bar{F}_{q,k}(x) = \frac{\sum_{i=k+1}^L\hat{F}_{t_{q,i}}(x)}{L-k}.  
	\end{equation*}
	Recall that $(t_{q,1}, \cdots, t_{q,K})$ is a random sub-sequence of 
	$\{(q-1)L+1, \cdots, qL \}$, and is used for ``exploration''.
	
	We first claim that similar result to \cref{lem:icml} in 
	\cref{app:one_shot} still holds for \cref{alg:bandit}.
	\begin{lemma}
		\label{lem:icml_extend_K'}
		If $F_t$ is monotone continuous DR-submodular and $L_2$-smooth, 
		$x_t^{(k+1)} = x_t^{(k)} + \frac{1}{K}(v_t^{(k)}-\underline{u})$ for $k 
		\in [K]$, where $v_t^{(k)}, x_t^{(k)} \in \constraint', \lb$ is the 
		lower bound of $\constraint'$, then
		\begin{equation*}
		\begin{split}
		F_t(x^*_\delta) -F_t(x_t^{(k+1)}) \leq & (1 - 1/K)[F_t(x^*_\delta) - 
		F_t(x_t^{(k)})] \\
		&-\frac{1}{K}[-\frac{1}{2\beta^{(k)}} \| \nabla F_t(x_t^{(k)}) - 
		d_t^{(k)} \|^2 - \frac{\beta^{(k)}D^2}{2} + \langle d_{t}^{(k)}, 
		v_t^{(k)} - x^*_\delta \rangle] + \frac{L_2D^2}{2K^2},   
		\end{split}
		\end{equation*}
		where $\{\beta^{(k)}\}$ is a sequence of positive parameters to be 
		determined.
	\end{lemma}
	
	\begin{proof}[Proof of \cref{lem:icml_extend_K'}]
		Since $F_t$ is $L_2$-smooth and $x_t^{(k+1)} = x_t^{(k)} + 
		\frac{1}{K}(v_t^{(k)}-\underline{u})$, we have
		\begin{equation}
		\label{eq:icml_extend_aux0}
		\begin{split}
		F_t(x_t^{(k+1)}) &\geq F_t(x_t^{(k)}) + \langle \nabla F_t(x_t^{(k)}), 
		x_t^{(k+1)} - x_t^{(k)} \rangle -\frac{L_2}{2} \| x_t^{(k+1)} 
		-x_t^{(k)} \|^2 \\
		& = F_t(x_t^{(k)}) + \langle \frac{1}{K} \nabla F_t(x_t^{(k)}), 
		v_t^{(k)} - \lb \rangle - \frac{L_2}{2K^2} \| v_t^{(k)} - \lb \|^2 \\
		&\geq F_t(x_t^{(k)}) + \frac{1}{K} \langle \nabla F_t(x_t^{(k)}), 
		v_t^{(k)} - \lb \rangle -\frac{L_2D^2}{2K^2}.
		\end{split}    
		\end{equation}
		
		We can rewrite the term $\langle \nabla F_t(x_t^{(k)}), v_t^{(k)} - \lb 
		\rangle$ as
		\begin{equation}
		\label{eq:icml_extend_aux1}
		\begin{split}
		\langle \nabla F_t(x_t^{(k)}), v_t^{(k)} - \lb \rangle &= \langle 
		\nabla F_t(x_t^{(k)}) - d_t^{(k)}, v_t^{(k)} \rangle + \langle 
		d_t^{(k)}, v_t^{(k)} \rangle - \langle \nabla F_t(x_t^{(k)}) , \lb 
		\rangle  \\
		&= \langle \nabla F_t(x_t^{(k)}) - d_t^{(k)}, v_t^{(k)} - x^*_\delta 
		\rangle + \langle \nabla F_t(x_t^{(k)}) - d_t^{(k)}, x^*_\delta \rangle 
		\\
		&\quad + \langle d_t^{(k)}, v_t^{(k)} \rangle - \langle \nabla 
		F_t(x_t^{(k)}) , \lb \rangle  \\
		&= \langle \nabla F_t(x_t^{(k)}) - d_t^{(k)}, v_t^{(k)} - x^*_\delta 
		\rangle + \langle \nabla F_t(x_t^{(k)}), x^*_\delta - \lb \rangle + 
		\langle d_t^{(k)}, v_t^{(k)} - x^*_\delta \rangle.
		\end{split}
		\end{equation}
		
		Denote $y^*_\delta = x^*_\delta - \lb, y_t^{(k)} = x_t^{(k)} - \lb$, 
		then $y^*_\delta \geq 0, y_t^{(k)} \geq 0$, by the definition of lower 
		bound $\lb$, and the fact $x^*_\delta, x_t^{(k)} \in \constraint'$.
		Since $F_t$ is monotone and is concave along non-negative directions, 
		we have
		\begin{equation}
		\label{eq:icml_extend_aux2}
		\begin{split}
		F_t(x^*_\delta) - F_t(x_t^{(k)}) &= F_t(y^*_\delta + \lb) - 
		F_t(y_t^{(k)} + \lb) \\
		&\leq F_t[ (y^*_\delta + \lb)\lor(y_t^{(k)}+\lb)] - F_t(y_t^{(k)}+\lb) 
		\\
		&\leq \langle \nabla F_t(y_t^{(k)}+\lb), [(y^*_\delta + 
		\lb)\lor(y_t^{(k)}+\lb)] - (y_t^{(k)}+\lb) \rangle \\
		&= \langle \nabla F_t(y_t^{(k)}+\lb), [(y^*_\delta + \lb) - 
		(y_t^{(k)}+\lb)] \lor 0 \rangle \\
		&= \langle \nabla F_t(y_t^{(k)}+\lb), (y^*_\delta  - y_t^{(k)}) \lor 0 
		\rangle \\
		&\leq \langle \nabla F_t(y_t^{(k)}+\lb), y^*_\delta \rangle \\
		&= \langle \nabla F_t(x_t^{(k)}), x^*_\delta - \lb \rangle.
		\end{split}
		\end{equation}
		
		Combine \cref{eq:icml_extend_aux1,eq:icml_extend_aux2}, we have
		\begin{equation}
		\label{eq:icml_extend_aux3}
		\langle \nabla F_t(x_t^{(k)}), v_t^{(k)} - \lb \rangle \geq \langle 
		\nabla F_t(x_t^{(k)}) - d_t^{(k)}, v_t^{(k)} - x^*_\delta \rangle + 
		[F_t(x^*_\delta) - F_t(x_t^{(k)})] + \langle d_t^{(k)}, v_t^{(k)} - 
		x^*_\delta \rangle.
		\end{equation}
		
		By Young's ineqaulity, we have
		\begin{equation}
		\label{eq:icml_extend_aux4}
		\begin{split}
		\langle \nabla F_t(x_t^{(k)}) - d_t^{(k)}, v_t^{(k)} - x^*_\delta 
		\rangle &\geq -\frac{1}{2\beta^{(k)}} \| \nabla F_t(x_t^{(k)}) - 
		d_t^{(k)}\|^2 - \frac{\beta^{(k)}}{2} \|v_t^{(k)} - x^*_\delta \|^2 \\
		&\geq -\frac{1}{2\beta^{(k)}} \| \nabla F_t(x_t^{(k)}) - d_t^{(k)}\|^2 
		- \frac{\beta^{(k)}D^2}{2}.    
		\end{split}
		\end{equation}
		
		Now combine 
		\cref{eq:icml_extend_aux0,eq:icml_extend_aux3,eq:icml_extend_aux4}, we 
		have
		\begin{equation*}
		\begin{split}
		F_t(x_t^{(k+1)}) &\geq \frac{1}{K} [-\frac{1}{2\beta^{(k)}} \| \nabla 
		F_t(x_t^{(k)}) - d_t^{(k)}\|^2 - \frac{\beta^{(k)}D^2}{2} 
		+ [F_t(x^*_\delta) - F_t(x_t^{(k)})] + \langle d_t^{(k)}, v_t^{(k)} - 
		x^*_\delta \rangle] \\
		&\quad + F_t(x_t^{(k)}) -\frac{L_2D^2}{2K^2}.
		\end{split}
		\end{equation*}
		Or, equivalently,
		\begin{equation*}
		\begin{split}
		F_t(x^*_\delta) -F_t(x_t^{(k+1)}) \leq & (1 - 1/K)[F_t(x^*_\delta) - 
		F_t(x_t^{(k)})] \\
		&-\frac{1}{K}[-\frac{1}{2\beta^{(k)}} \| \nabla F_t(x_t^{(k)}) - 
		d_t^{(k)} \|^2 - \frac{\beta^{(k)}D^2}{2} + \langle d_{t}^{(k)}, 
		v_t^{(k)} - x^*_\delta \rangle] + \frac{L_2D^2}{2K^2},   
		\end{split}
		\end{equation*}
	\end{proof}
	
	Since $\hat{F}_t$ is monotone continuous DR-submodular and $L_2$-smooth for 
	all $t$, with \cref{lem:icml_extend_K'}, and repeating the proof of 
	\cref{lem:regret_decomp} in \cref{app:one_shot}, we have
	\begin{equation*}
	\begin{split}
	\expect[(1-1/e)\bar{F}_{q,0}(x^*_\delta) - \bar{F}_{q,0}(x_q)] \leq & 
	\expect[ \frac{1}{K}\sum_{k=1}^K[\frac{1}{2\beta^{(k)}} 
	\Delta_q^{(k)} + \frac{\beta^{(k)}D^2}{2}]] + \frac{L_2D^2}{2K} \\
	& \quad + 1/K \sum_{k=1}^K (1-1/K)^{K-k} \expect [\langle d_{q}^{(k)}, 
	x^*_\delta -  v_q^{(k)} \rangle ]
	\end{split}
	\end{equation*}
	where $\Delta_q^{(k)} = \| \nabla \bar{F}_{q,k-1}(x_q^{(k)}) - 
	d_q^{(k)}\|^2$.
	
	Therefore, we have
	\begin{equation}
	\begin{split}
	& \expect[\sum_{q=1}^Q \sum_{i=1}^L 
	[(1-1/e)\hat{F}_{t_{q,i}}(x^*_\delta)-\hat{F}_{t_{q,i}}(x_q)]]\\
	=& \sum_{q=1}^Q L \expect[(1-1/e)\bar{F}_{q,0}(x^*_\delta) -
	\bar{F}_{q,0}(x_q)] \\
	=& \expect[\frac{L}{K}\sum_{q=1}^Q \sum_{k=1}^K 
	\frac{\Delta_q^{(k)}}{2\beta^{(k)}}] + \frac{LQ}{K}\sum_{k=1}^K 
	\frac{\beta^{(k)}D^2}{2} + \frac{LQL_2D^2}{2K} \\
	& \qquad + \frac{L}{K}\sum_{k=1}^K (1-1/K)^{K-k} \sum_{q=1}^Q \expect 
	[\langle d_{q}^{(k)}, 
	x^*_\delta -  v_q^{(k)} \rangle ] \\
	\leq & \expect[\frac{L}{K}\sum_{q=1}^Q \sum_{k=1}^K 
	\frac{\Delta_q^{(k)}}{2\beta^{(k)}}] + \frac{LQ}{K}\sum_{k=1}^K 
	\frac{\beta^{(k)}D^2}{2} + \frac{LQL_2D^2}{2K} \\
	& \qquad + \frac{L}{K} \sum_{k=1}^K 1 \cdot \mathcal{R}_Q^{\mathcal{E}} \\
	\leq & \expect[\frac{L}{K}\sum_{q=1}^Q \sum_{k=1}^K 
	\frac{\Delta_q^{(k)}}{2\beta^{(k)}}] + \frac{LQ}{K}\sum_{k=1}^K 
	\frac{\beta^{(k)}D^2}{2} + \frac{LQL_2D^2}{2K} + 
	L\mathcal{R}_Q^{\mathcal{E}}.
	\end{split}    
	\end{equation}
	
	Then we have
	\begin{equation}
	\label{eq:bandit_regret_decomp}   
	\begin{split}
	\expect[\mathcal{R}_T] &\leq (1-1/e)c_1L_1T\delta^\gamma + 
	(2-1/e)L_1T\delta + 2M_1QK\\
	&\quad + \expect[\frac{L}{K}\sum_{q=1}^Q \sum_{k=1}^K 
	\frac{\Delta_q^{(k)}}{2\beta^{(k)}}] + \frac{LQ}{K}\sum_{k=1}^K 
	\frac{\beta^{(k)}D^2}{2} + \frac{LQL_2D^2}{2K} + 
	L\mathcal{R}_Q^{\mathcal{E}}.  
	\end{split}
	\end{equation}
	
	Note $\mathcal{R}_Q^{\mathcal{E}}$ is the regret of the online linear 
	maximization oracle $\mathcal{E}$ at horizon $Q$, which is of 
	order $O(\sqrt{Q})$. So in order to get an upper bound for the expected 
	regret of \cref{alg:bandit}, the key is to bound $\expect[\Delta_q^{(k)}]$. 
	Here, we have an analogue of \cref{lem:delta_decomp} in \cref{app:one_shot}:
	
	\begin{lemma}
		\label{lem:bandit_delta_decomp}
		Under the setting of \cref{thm:bandit}, we have
		\begin{equation*}
		\expect[\Delta_q^{(k)}] \leq \rho_k^2 \sigma^2 + (1-\rho_k)^2 
		\expect[\Delta_q^{(k-1)}] + (1-\rho_k)^2 \frac{G}{(k+2)^2} + 
		(1-\rho_k)^2\left[ \frac{G}{\alpha_k (k+2)^2} + 
		\alpha_k\expect[\Delta _q^{(k-1)}] \right], 
		\end{equation*}
		where $\{\alpha_k\}$ is a sequence of positive parameters to be 
		determined, $\sigma^2 = L_1^2 + \frac{d^2M_1^2}{\delta^2}$, 
		$G=[3L_2R+2L_1]^2$.
	\end{lemma}
	
	\begin{proof}[Proof of \cref{lem:bandit_delta_decomp}]
		First, the decomposition of $\Delta_q^{(k)}$ \cref{eq:delta} still 
		holds, with $\tnabla F_{t_{q,k}}(x_q^{(k)})$ replaced by $g_{q,k}$.
		
		We also denote $\mathcal{F}_{q,k}$ to be the $\sigma$-field generated 
		by 
		$t_{q,1}, t_{q,2}, \cdots, t_{q,k}$. Since 
		$\expect[g_{q,k}|\mathcal{F}_{q,k}] = \nabla 
		\hat{F}_{t_{q,k}}(x_q^{(k)})|\mathcal{F}_{q,k}$, we have 
		$\expect[g_{q,k}|\mathcal{F}_{q,k-1}] = \nabla 
		\bar{F}_{q,k-1}(x_q^{(k)})|\mathcal{F}_{q,k-1}$. Then by law of 
		iterated 
		expectations, we can get the results similar to 
		\cref{eq:decomp_11,eq:decomp_12,eq:decomp_13,eq:decomp_14,eq:decomp_21}.
		 
		Precisely, we have:
		\begin{equation*}
		\begin{split}
		\expect[\expect[\| \nabla \bar{F}_{q,k-1}(x_q^{(k)}) - \nabla 
		\hat{F}_{t_{q,k}}(x_q^{(k)}) \|^2 |\mathcal{F}_{q,k-1}]] &= 
		\expect[\text{Var}(\nabla 
		\hat{F}_{t_{q,k}}(x_q^{(k)})|\mathcal{F}_{q,k-1})] \\
		&\leq \expect[ \| \nabla 
		\hat{F}_{t_{q,k}}(x_q^{(k)})\|^2] \\
		&\leq L_1^2,    
		\end{split}
		\end{equation*}
		\begin{equation*}
		\begin{split}
		\expect[\expect[ \|\nabla \hat{F}_{t_{q,k}}(x_q^{(k)}) - g_{q,k} \|^2 
		|\mathcal{F}_{q,k-1}]] 
		=& \expect[ \|\nabla \hat{F}_{t_{q,k}}(x_q^{(k)}) - g_{q,k} \|^2]\\
		=& \expect[\expect[ \|\nabla \hat{F}_{t_{q,k}}(x_q^{(k)}) - g_{q,k} 
		\|^2 |\mathcal{F}_{q,k}]] \\ 
		=&\expect[\text{Var}(g_{q,k}|\mathcal{F}_{q,k})] \\
		\leq& \frac{d^2M_1^2}{\delta^2},  
		\end{split}
		\end{equation*}
		and
		\begin{equation*}
		\expect[\expect[\langle \nabla \bar{F}_{q,k-1}(x_q^{(k)}) - \nabla 
		\hat{F}_{t_{q,k}}(x_q^{(k)}),\nabla \hat{F}_{t_{q,k}}(x_q^{(k)}) - 
		g_{q,k} \rangle|\mathcal{F}_{q,k-1}]] = 0. 
		\end{equation*}
		
		Thus we have
		\begin{equation}
		\label{eq:bandit_decomp1}
		\begin{split}
		&\expect [\| \nabla \bar{F}_{q,k-1}(x_q^{(k)}) - g_{q,k} \|^2] \\
		=& \expect[\expect[\| \nabla \bar{F}_{q,k-1}(x_q^{(k)}) - g_{q,k} \|^2  
		|\mathcal{F}_{q,k-1}]]  \\
		=& \expect[\expect[\| \nabla \bar{F}_{q,k-1}(x_q^{(k)}) - \nabla 
		\hat{F}_{t_{q,k}}(x_q^{(k)}) \|^2  + \|\nabla 
		\hat{F}_{t_{q,k}}(x_q^{(k)}) - 
		g_{q,k} \|^2 \\
		& \quad + 2 \langle \nabla \bar{F}_{q,k-1}(x_q^{(k)}) - \nabla
		\hat{F}_{t_{q,k}}(x_q^{(k)}),\nabla \hat{F}_{t_{q,k}}(x_q^{(k)}) - 
		g_{q,k} \rangle|\mathcal{F}_{q,k-1}]]\\
		\leq& L_1^2 + \frac{d^2M_1^2}{\delta^2}\\
		\triangleq& \sigma^2.
		\end{split}   
		\end{equation}
		
		We also have the results similar to \cref{eq:decomp_22,eq:decomp_23}:
		\begin{equation}
		\label{eq:bandit_decomp2}
		\expect[\langle \nabla \bar{F}_{q,k-1}(x_q^{(k)}) - g_{q,k}, \nabla 
		\bar{F}_{q,k-1}(x_q^{(k)}) - \nabla \bar{F}_{q,k-2}(x_q^{(k-1)}) 
		\rangle] = 0,    
		\end{equation}
		and
		\begin{equation}
		\label{eq:bandit_decomp3}
		\expect[\langle \nabla \bar{F}_{q,k-1}(x_q^{(k)}) - g_{q,k}, \nabla 
		\bar{F}_{q,k-2}(x_q^{(k-1)}) - d_q^{(k-1)} \rangle] = 0.   
		\end{equation}
		
		Also, by Young's Inequality, we have 
		\begin{equation}
		\label{eq:bandit_decomp4}
		\begin{split}
		\langle \nabla \bar{F}_{q,k-1}(x_q^{(k)}) - &\nabla  
		\bar{F}_{q,k-2}(x_q^{(k-1)}) , \nabla \bar{F}_{q,k-2}(x_q^{(k-1)}) - 
		d_q^{(k-1)} \rangle \\
		&\leq \frac{1}{2\alpha_k} \|\nabla \bar{F}_{q,k-1}(x_q^{(k)}) - 
		\nabla \bar{F}_{q,k-2}(x_q^{(k-1)}) \|^2 + 
		\frac{\alpha_k}{2}\Delta_q^{(k-1)}.  
		\end{split}
		\end{equation}
		
		Now we turn to bound $\| \nabla \bar{F}_{q,k-1}(x_q^{(k)}) - \nabla 
		\bar{F}_{q,k-2}(x_q^{(k-1)})\|^2 \triangleq z_{q,k}^2.$ Actually, we 
		have
		\begin{equation*}
		\begin{split}
		\expect[z_{q,k}^2] &= \expect[\expect[\| \nabla 
		\bar{F}_{q,k-1}(x_q^{(k)}) - \nabla 
		\bar{F}_{q,k-2}(x_q^{(k-1)})\|^2 |\mathcal{F}_{q,k-2}]]  \\
		&= \expect[\expect[\| \frac{\sum_{i=k}^L \nabla 
			\hat{F}_{t_{q,i}}(x_q^{(k)})}{L-k+1} - \frac{\sum_{i=k-1}^L \nabla 
			\hat{F}_{t_{q,i}}(x_q^{(k-1)})}{L-k+2}\|^2 |\mathcal{F}_{q,k-2}]] \\
		&= \expect[ \expect[ \| \frac{\sum_{i=k}^L \nabla 
		\hat{F}_{t_{q,i}}(x_q^{(k)})-\nabla 
		\hat{F}_{t_{q,i}}(x_q^{(k-1)})}{L-k+2} + \frac{\sum_{i=k}^L \nabla 
		\hat{F}_{t_{q,i}}(x_q^{(k)})}{(L-k+1)(L-k+2)} \\
		&\quad - \frac{\nabla \hat{F}_{t_{q,k-1}}(x_q^{(k-1)})}{L-k+2} \|^2 
		|\mathcal{F}_{q,k-2}]] \\
		&\leq \expect[ \expect[ (\sum_{i=k}^L \| \frac{\nabla 
			\hat{F}_{t_{q,i}}(x_q^{(k)})-\nabla 
			\hat{F}_{t_{q,i}}(x_q^{(k-1)})}{L-k+2}\| + 
		\sum_{i=k}^L \| \frac{\nabla  
		\hat{F}_{t_{q,i}}(x_q^{(k)})}{(L-k+1)(L-k+2)} \| \\
		& \quad + \| \frac{\nabla \hat{F}_{t_{q,k-1}}(x_q^{(k-1)})}{L-k+2} 
		\|)^2  |\mathcal{F}_{q,k-2}]],
		\end{split}
		\end{equation*}
		where the inequality comes from the Triangle Inequality of norms.
		
		Recall the update rule where $x_q^{(k)} = x_q^{(k-1)} + 
		\frac{1}{K}(v_q^{(k-1)}-\lb)$ and that $\hat{F}_t$ is $L_2$-smooth, we 
		have 
		\begin{equation*}
		\|\nabla \hat{F}_{t_{q,i}}(x_q^{(k)})-\nabla 
		\hat{F}_{t_{q,i}}(x_q^{(k-1)}) \| \leq L_2\frac{\| 
		v_q^{(k-1)}-\lb\|}{K} \leq \frac{L_2R}{K}.    
		\end{equation*}
		
		Also by \cref{assump_on_f}, $\| \nabla F_{t_{q,i}}(x) \| \leq 
		L_1$ for all $x \in \constraint$, thus $\| \nabla 
		\hat{F}_{t_{q,i}}(x_q^{(k)}) \| \leq L_1$, $\|\nabla 
		\hat{F}_{t_{q,k-1}}(x_q^{(k-1)})\| \leq L_1$.  Therefore, we have
		\begin{equation*}
		\begin{split}
		\expect[z_{q,k}^2] & \leq  [ (L-k+1) \frac{L_2R}{K}\frac{1}{L-k+2}  + 
		(L-k+1)\frac{L_1}{(L-k+1)(L-k+2)} + \frac{L_1}{L-k+2}]^2\\
		&\leq \left( \frac{L-k+1}{L-k+2}\frac{L_2R}{K}+\frac{2L_1}{L-k+2} 
		\right)^2. 
		\end{split}
		\end{equation*}
		
		Since we assume $L \gg K$, 
		we can always choose $L, K$ such that $L \geq 2K$. So we have 
		$\frac{2L_1}{L-k+2} \leq \frac{2L_1}{2K-k+2} \leq \frac{2L_1}{K+2} \leq 
		\frac{2L_1}{k+2}$. Also, $\frac{L-k+1}{L-k+2}\frac{L_2R}{K} \leq 
		\frac{L_2R}{K} = \frac{K+2}{K}\frac{L_2R}{K+2} \leq 3\frac{L_2R}{K+2} 
		\leq \frac{3L_2R}{k+2}.$
		
		Therefore, we have 
		\begin{equation}
		\label{eq:bandit_decomp5}
		\begin{split}
		\expect[z_{q,k}^2] & \leq \left( \frac{3L_2R}{k+2} + 
		\frac{2L_1}{k+2}\right)^2 \\
		& = \left(\frac{3L_2R+2L_1}{k+2} \right)^2 \\
		& \triangleq \frac{G}{(k+2)^2}.
		\end{split}
		\end{equation}
		
		Combining 
		\cref{eq:bandit_decomp1,eq:bandit_decomp2,eq:bandit_decomp3,eq:bandit_decomp4,eq:bandit_decomp5},
		we have
		\begin{equation*}
		\expect[\Delta_q^{(k)}] \leq \rho_k^2 \sigma^2 + (1-\rho_k)^2 
		\expect[\Delta_q^{(k-1)}] + (1-\rho_k)^2 \frac{G}{(k+2)^2} + 
		(1-\rho_k)^2\left[ \frac{G}{\alpha_k (k+2)^2} + 
		\alpha_k\expect[\Delta _q^{(k-1)}] \right].
		\end{equation*}
	\end{proof}
	
	Applying \cref{lem:bandit_delta_decomp} and setting $\alpha_k = 
	\frac{\rho_k}{2}, \forall k \in {1, 2, \cdots, K}$, we have
	\begin{equation*}
	\begin{split}
	\expect[\Delta_q^{(k)}] &\leq \rho_k^2 \sigma^2 + (1-\rho_k)^2 
	\expect[\Delta_q^{(k-1)}] + (1-\rho_k)^2 \frac{G}{(k+2)^2}\\
	&\quad + (1-\rho_k)^2\left[ \frac{G}{\alpha_k (k+2)^2} + 
	\alpha_k\expect[\Delta _q^{(k-1)}] \right] \\
	&= \rho_k^2 \sigma^2 + \frac{G}{(k+2)^2}(1-\rho_k)^2 \left( 1+ 
	\frac{2}{\rho_k} \right) + 
	\expect[\Delta_q^{(k-1)}](1-\rho_k)^2\left(1+\frac{\rho_k}{2}\right).
	\end{split}    
	\end{equation*}
	
	Note that if $0 < \rho_k \leq 1$, then we have 
	$$(1-\rho_k)^2\left(1+\frac{2}{\rho_k} \right) \leq 
	\left(1+\frac{2}{\rho_k} \right)$$
	and
	$$(1-\rho_k)^2\left(1+\frac{\rho_k}{2} \right) \leq (1 - \rho_k).$$
	
	So in this case, we have 
	\begin{equation}
	\label{eq:bandit_delta_iter}
	\expect[\Delta_q^{(k)}] \leq \rho_k^2 \sigma^2 + 
	\frac{G}{(k+2)^2}\left( 1+ \frac{2}{\rho_k} \right) + 
	\expect[\Delta_q^{(k-1)}](1-\rho_k).
	\end{equation}
	
	\begin{lemma}
		\label{lem:bandit_delta_final}
		Under the setting of \cref{thm:bandit}, 
		we have
		\begin{equation*}
		\expect[\Delta_q^{(k)}] \leq \frac{N_0}{(k+3)^{2/3}}, \forall k \in 
		[K],    
		\end{equation*}
		where $N_0 = 4^{2/3}(2\sigma^2+G)$.
	\end{lemma}	
	
	\begin{proof}[Proof of \cref{lem:bandit_delta_final}]
		Since $\rho_k = \frac{2}{(k+2)^{2/3}}$, we have $0 < \rho_k \leq 1$, 
		and 
		\begin{equation*}
		\begin{split}
		\expect[\Delta_q^{(k)}] &\leq \frac{4\sigma^2}{(k+2)^{4/3}} + 
		\frac{G}{(k+2)^2}[1+(k+2)^{2/3}] + \expect[\Delta_q^{(k-1)}]\left( 1 - 
		\frac{2}{(k+2)^{2/3}} \right)  \\
		&\leq \frac{4\sigma^2}{(k+2)^{4/3}} + \frac{G}{(k+2)^{4/3}} + 
		\frac{G}{(k+2)^{4/3}} + \expect[\Delta_q^{(k-1)}]\left( 1 - 
		\frac{2}{(k+2)^{2/3}} \right)  \\
		&= \frac{4\sigma^2+2G}{(k+2)^{4/3}} + \expect[\Delta_q^{(k-1)}]\left( 1 
		- \frac{2}{(k+2)^{2/3}} \right) \\
		& \leq \frac{\frac{4^{2/3}}{2}(4\sigma^2+2G)}{(k+2)^{4/3}} + 
		\expect[\Delta_q^{(k-1)}]\left( 1 - \frac{2}{(k+2)^{2/3}} \right) \\
		& = \frac{4^{2/3}(2\sigma^2+G)}{(k+2)^{4/3}} + 
		\expect[\Delta_q^{(k-1)}]\left( 1 - \frac{2}{(k+2)^{2/3}} \right) \\
		&\triangleq \frac{N_0}{(k+2)^{4/3}} +  \expect[\Delta_q^{(k-1)}]\left( 
		1 - \frac{2}{(k+2)^{2/3}} \right).
		\end{split}
		\end{equation*}
		
		Recall that $\Delta_q^{(k)} = \|\nabla \bar{F}_{q,k-1}(x_q^{(k)}) - 
		d_q^{(k)} \|^2$, and thus
		\begin{equation*}
		\begin{split}
		\Delta_q^{(1)} &= \|\nabla \bar{F}_{q,0}(\lb) - d_q^{(1)}  \|^2 \\
		&= \| \frac{\sum_{i=1}^L \nabla \hat{F}_{t_{q,i}}(\lb)}{L} - 
		\frac{2}{3^{2/3}} g_{q,1}) \|^2 \\
		&\leq \left( \sum_{i=1}^{L} \| \frac{\nabla 
		\hat{F}_{t_{q,i}}(\lb)}{L}\| + \| 
		\frac{2}{3^{2/3}} g_{q,1}\| \right)^2 \\
		&\leq \left(L\frac{L_1}{L} + \frac{2}{3^{2/3}} \frac{d}{\delta}M_1 
		\right)^2 \\
		&\leq (L_1+\frac{d}{\delta}M_1)^2. 
		\end{split}
		\end{equation*}
		
		Now we claim that $\expect[\Delta_q^{(k)}] \leq 
		\frac{N_0}{(k+3)^{2/3}}$ for any $k \in [K]$. We prove it by induction. 
		When $k = 1$, we have 
		$$\frac{N_0}{(1+3)^{2/3}} = 2\sigma^2+G \geq 2\sigma^2 = 
		2(L_1^2+\frac{d^2M_1^2}{\delta^2}) \geq (L_1 + \frac{dM_1}{\delta})^2 
		\geq \Delta_q^{(1)},$$ 
		where the second inequality holds since $2(a^2+b^2) \geq (a+b)^2.$
		
		Assume the statement holds for $k - 1$, \emph{i.e.}, 
		$\expect[\Delta_q^{(k-1)}] \leq \frac{N_0}{(k+2)^{2/3}}$, then
		\begin{equation*}
		\begin{split}
		\expect[\Delta_q^{(k)}] &\leq \frac{N_0}{(k+2)^{4/3}} + 
		\expect[\Delta_q^{(k-1)}]\left( 1 - \frac{2}{(k+2)^{2/3}} \right) \\
		&\leq \frac{N_0}{(k+2)^{4/3}} + \frac{N_0}{(k+2)^{2/3}}\left( 1 - 
		\frac{2}{(k+2)^{2/3}} \right) \\
		&= \frac{N_0[(k+2)^{2/3} - 1]}{(k+2)^{4/3}}. 
		\end{split}    
		\end{equation*}
		
		Since $(k+3)^2 = k^2+6k+9 \leq k^2+4k+4+1+3(k+2) \leq (k+2)^2 + 1 + 
		3(k+2)^{4/3} + 3(k+2)^{2/3} = [(k+2)^{2/3}+1]^3$, by taking the cube 
		roots of both sides, we have $(k+3)^{2/3} \leq (k+2)^{2/3}+1$, which 
		implies that $[(k+2)^{2/3}-1](k+3)^{2/3} \leq 
		[(k+2)^{2/3}-1][(k+2)^{2/3}+1] \leq (k+2)^{4/3}$, \emph{i.e.}, 
		$\frac{(k+2)^{2/3}-1}{(k+2)^{4/3}} \leq \frac{1}{(k+3)^{2/3}}$. Thus we 
		have 
		\begin{equation*}
		\expect[\Delta_q^{(k)}] \leq \frac{N_0}{(k+3)^{2/3}}, \forall k \in 
		[K].    
		\end{equation*}
	\end{proof}
	
	Recall that in \cref{eq:bandit_regret_decomp}, we have
	\begin{equation*}
	\begin{split}
	\expect[\mathcal{R}_T] &\leq (1-1/e)c_1L_1T\delta^\gamma + 
	(2-1/e)L_1T\delta + 2M_1QK \\
	&\quad + \expect[\frac{L}{K}\sum_{q=1}^Q \sum_{k=1}^K 
	\frac{\Delta_q^{(k)}}{2\beta^{(k)}}]
	+ \frac{LQ}{K}\sum_{k=1}^K \frac{\beta^{(k)}D^2}{2} + \frac{LQL_2D^2}{2K} + 
	L\mathcal{R}_Q^{\mathcal{E}}.  
	\end{split}
	\end{equation*}
	
	So if we set $\beta^{(k)} = \frac{1}{\delta(k+3)^{1/3}}$, then by 
	\cref{lem:bandit_delta_final}, we have
	\begin{equation*}
	\sum_{k=1}^{K}\frac{\expect[\Delta_q^{(k)}]}{\beta^{(k)}} \leq 
	\sum_{k=1}^{K} \frac{\delta N_0}{(k+3)^{1/3}} 
	\leq \sum_{k=1}^{K}\frac{\delta N_0}{k^{1/3}}  \leq 
	\int_0^{K}\frac{\delta N_0}{x^{1/3}}\mathrm{d}x 
	=\frac{3\delta N_0}{2}K^{2/3}.   
	\end{equation*}
	
	Similarly, 
	$$\sum_{k=1}^{K}\beta^{(k)} = \sum_{k=1}^{K} \frac{1}{\delta (k+3)^{1/3}} 
	\leq \frac{3K^{2/3}}{2\delta}.$$
	
	Therefore, we have
	\begin{equation*}
	\expect[\mathcal{R}_T] \leq (1-1/e)c_1L_1T\delta^\gamma + (2-1/e)L_1T\delta 
	+ 2M_1QK
	+ \frac{3\delta N_0LQ}{4K^{1/3}}+\frac{3D^2LQ}{4\delta K^{1/3}} + 
	\frac{LQL_2D^2}{2K} + L\mathcal{R}_Q^{\mathcal{E}}.  
	\end{equation*}
	
	By setting $\delta = c_2T^{-\frac{1}{3+6\min\{1,\gamma\}}}, Q = 
	T^{\frac{2\min\{1,\gamma\}}{3+6\min\{1,\gamma\}}}, L = 
	T^{\frac{3+4\min\{1,\gamma\}}{3+6\min\{1,\gamma\}}}, K = 
	T^{\frac{1+\min\{1,\gamma\}}{1+2\min\{1,\gamma\}}}$, and recall that 
	$\mathcal{R}_Q^{\mathcal{E}} 
	\leq C\sqrt{Q} = CT^{\frac{\min\{1,\gamma\}}{3+6\min\{1,\gamma\}}}$, $N_0 = 
	4^{2/3}(2\sigma^2 +G) = 4^{2/3}(2L_1^2+\frac{2d^2M_1^2}{\delta^2}+G)$, 
	where $G = (3L_2R+2L_1)^2$ is a constant, we have
	\begin{equation*}
	\begin{split}
	\expect[\mathcal{R}_T] \leq& (1-1/e)c_1c_2^\gamma L_1 
	T^{1-\frac{\gamma}{3+6\min\{1,\gamma\}}} + 
	(2-1/e)c_2L_1T^{1-\frac{1}{3+6\min\{1,\gamma\}}} + 
	2M_1T^{\frac{3+5\min\{1,\gamma\}}{3+6\min\{1,\gamma\}}} \\
	&\quad + 
	\frac{3\cdot4^{2/3}c_2(2L_1^2+G)}{4}T^{\frac{1+5\min\{1,\gamma\}}{3+6\min\{1,\gamma\}}}
	 + 
	\frac{3\cdot4^{2/3}d^2M_1^2}{2c_2}T^{\frac{3+5\min\{1,\gamma\}}{3+6\min\{1,\gamma\}}}
	 + \frac{3D^2}{4c_2}T^{\frac{3+5\min\{1,\gamma\}}{3+6\min\{1,\gamma\}}} \\
	&\quad + \frac{L_2D^2}{2}T^{\frac{\min\{1,\gamma\}}{1+2\min\{1,\gamma\}}} + 
	CT^{\frac{3+5\min\{1,\gamma\}}{3+6\min\{1,\gamma\}}}  \\
	\leq & (1-1/e)c_1c_2^\gamma L_1 
	T^{1-\frac{\min\{1,\gamma\}}{3+6\min\{1,\gamma\}}} + 
	(2-1/e)c_2L_1T^{1-\frac{\min\{1,\gamma\}}{3+6\min\{1,\gamma\}}} \\
	&\quad + 
	\left[ 2M_1+\frac{3\cdot4^{2/3}d^2M_1^2}{2c_2}+\frac{3D^2}{4c_2}+C 
	\right]T^{\frac{3+5\min\{1,\gamma\}}{3+6\min\{1,\gamma\}}} \\
	& \quad + 
	\frac{3\cdot4^{2/3}c_2(2L_1^2+G)}{4}T^{\frac{1+5\min\{1,\gamma\}}{3+6\min\{1,\gamma\}}}
	 + \frac{L_2D^2}{2}T^{\frac{\min\{1,\gamma\}}{1+2\min\{1,\gamma\}}} \\
	=& \left[(1-1/e)c_1c_2^\gamma L_1 + (2-1/e)c_2L_1+ 
	2M_1+\frac{3\cdot4^{1/6}d^2M_1^2}{c_2}+\frac{3D^2}{4c_2}+C 
	\right]T^{\frac{3+5\min\{1,\gamma\}}{3+6\min\{1,\gamma\}}} \\
	& \quad + 
	\frac{3c_2[2L_1^2+(3L_2R+2L_1)^2]}{4^{1/3}}T^{\frac{1+5\min\{1,\gamma\}}{3+6\min\{1,\gamma\}}}
	 + \frac{L_2D^2}{2}T^{\frac{\min\{1,\gamma\}}{1+2\min\{1,\gamma\}}}.
	\end{split}
	\end{equation*}
\end{proof}

\subsection{Down-closed Constraint Set}\label{app:donw-closed_K}

\begin{proof}[Proof of \cref{thm:downclose_bandit}]
	Since $\constraint$ satisfies \cref{assump_on_K} and is down-closed, 
	$\alpha = \frac{(\sqrt{d}+1)\delta}{r} = 
	\frac{\sqrt{d}+1}{\sqrt{d}+2}T^{-1/9} < 1$, by 
	\cref{lem:shrink_down_closed}, we have \cref{assump_on_constraint'} holds 
	with $c_1 = \sqrt{d}(\frac{R}{r}+1)+\frac{R}{r}, \gamma = 1, \underline{u} 
	= \delta \one$. Then by applying \cref{thm:bandit} directly, we can prove 
	\cref{thm:downclose_bandit}.  
\end{proof}

\section{Proof of \cref{lem:sampling_scheme}}\label{app:sampling_scheme}
\begin{proof}[Proof of \cref{lem:sampling_scheme}]
	We give an example of the matroids which satisfy 
	\cref{lem:sampling_scheme}. Let $\Omega = \{1, 2\}$, the matroid $\matroid 
	= 
	\{\varnothing, \{1\}, \{2\} \}$. Define set function
	\begin{equation*}
	f(X) = \begin{cases}
	0, & X = \varnothing; \\
	a, & X =\{ 1\};\\
	b, & X = \{2\}, \text{ or } X = \{1,2\};
	\end{cases}    
	\end{equation*}
	where $b >a > 0$. It can be verified that $f$ is submodular and its 
	multilinear extension $F(x) = ax_1 + bx_2 - ax_1x_2$, where $x = (x_1, x_2) 
	\in [0,1]^2$.
	
	Suppose that
	\begin{equation*}
	\round(x) =
	\begin{cases}
	\{1\}, & \text{with probability } p_1(x); \\
	\{2\}, & \text{with probability } p_2(x); \\
	\varnothing, & \text{with probability } p_3(x).
	\end{cases}
	\end{equation*}
	Then the assumption $F(x) = \expect[f(\round(x)]$ implies $F(x) = 
	p_1(x)\cdot a 
	+ p_2(x) \cdot b, \forall b > a >0$. So we have $p_1(x) = x_1-x_1x_2, 
	p_2(x)=x_2$.
	
	However, if we define $f$ in another way:
	\begin{equation*}
	f(X) = \begin{cases}
	0, & X = \varnothing; \\
	b, & X = \{2\}; \\
	a, & X= \{1\}, \text{ or } X=\{1,2\};
	\end{cases}    
	\end{equation*}
	where $a > b >0$. Then it can be also verified that $f$ is submodular and 
	its multilinear extension $F(x) = ax_1 + bx_2 - bx_1x_2$, where $x = (x_1, 
	x_2) \in [0,1]^2$.
	
	Again, suppose that
	\begin{equation*}
	\round(x) =
	\begin{cases}
	\{1\}, & \text{with probability } p_1(x); \\
	\{2\}, & \text{with probability } p_2(x); \\
	\varnothing, & \text{with probability } p_3(x).
	\end{cases}
	\end{equation*}
	Then the assumption $F(x) = \expect[f(\round(x)]$ implies $F(x) = 
	p_1(x)\cdot a 
	+ p_2(x) \cdot b, \forall a > b >0$. So we have $p_1(x) = x_1, 
	p_2(x)=x_2-x_1x_2$.
	
	Therefore, for different functions $f$'s, we have different sampling 
	schemes 
	$\round(\cdot)$'s, which are subject to the matroid $\matroid$ constraint, 
	and satisfy 
	$F(x) = \expect[f(\round(x)]$, \ie, the sampling scheme does depend on the 
	function. So there does not exist a sampling scheme $\round: [0,1]^d \to 
	\matroid$, 
	which satisfies $\expect[f(\round(x))] = F(x), \forall x \in [0,1]^d$, and 
	does 
	not depend on the submodular set function $f$,
\end{proof}

\section{Proof of \cref{thm:discrete_bandit}}\label{app:discrete_bandit}
Since \cref{alg:discrete_bandit} applies \cref{alg:bandit} on the multilinear 
extension $F_t$ of $f_t$, a prerequisite is that 
\cref{assump_on_oracle,assump_on_f,assump_on_K,bandit_assump_on_f,assump_on_discrete_bandit_K}
 all hold. 
The constraint set $\constraint$ is a polytope in $[0,1]^d$ that is convex and 
compact and contains $ 0 $. So \cref{assump_on_K} holds. Additionally, we have 
the diameter $D = \sup_{x,y \in 
	\constraint}\|x-y\| \leq \sqrt{d}$ and the radius $R = \sup_{x \in 
	\constraint} \|x\| \leq \sqrt{d}$.  

Since each objective function $f_t$ is monotone submodular, its multilinear 
extension $F_t$ is monotone and continuous DR-submodular 
\citep{calinescu2011maximizing}. If $\sup_{X \subseteq 
	\Omega}|f_t(X)| \leq M$, then \cref{bandit_assump_on_f} holds for $F_t$
	automatically, and the following lemma shows that its multilinear 
	extension $F_t$ is Lipschitz and smooth, which entails \cref{assump_on_f}.

\begin{lemma}[Lemma 4 in \citep{chen2019black}]
	\label{lem:discrete_to_continuous}
	For a submodular set function $f$ with $\sup_{X \subseteq \Omega} |f(X)| 
	\le 
	M$, its multilinear extension $F$ is $(2M\sqrt{d})$-Lipschitz and  
	$(4M\sqrt{d(d-1)})$-smooth.
\end{lemma}

In summary, we only need 
\cref{assump_on_oracle,assump_on_discrete_bandit_K,assump_on_discrete_bandit_f}.
 Now we turn to prove \cref{thm:discrete_bandit}.


\begin{proof}[Proof of \cref{thm:discrete_bandit}]
	We first define $X^* = \argmax_{X \in \matroid}\sum_{t=1}^T f_t(X)$, the 
	corresponding fractional solution is $\tilde{x} \in \constraint$, 
	\emph{i.e.}, 
	\begin{equation}
	\label{eq:discrete_aux1}
	f_t(X^*) = F_t(\tilde{x}),    
	\end{equation}
	where $F_t$ is the multilinear extension of $f_t$. We also define $x^* = 
	\argmax_{x \in \constraint} \sum_{t=1}^T F_t(x), x_\delta^* = \argmax_{x 
	\in \constraint'} \sum_{t=1}^T F_t(x)$. The $(1-1/e)$-regret with horizon 
	$T$ is
	\begin{equation}
	\label{eq:discrete_regret}
	\mathcal{R}_T = \sum_{t=1}^T[(1-1/e)f_t(X^*) - f_t(Y_t)\mathds{1}_{Y_t \in 
	\matroid}].   
	\end{equation}
	We have
	\begin{equation}
	\label{eq:discrete_aux2}
	\begin{split}
	\sum_{t=1}^T f_t(Y_t)\mathds{1}_{Y_t \in \matroid} &= 
	\sum_{q=1}^Q\sum_{i=1}^L f_{t_{q,i}}(Y_{t_{q,i}})\mathds{1}_{Y_{t_{q,i}} 
	\in \matroid} \\
	&= \sum_{q=1}^Q\sum_{i=K+1}^L f_{t_{q,i}}(Y_{t_{q,i}}) + 
	\sum_{q=1}^Q\sum_{k=1}^K F_{t_{q,k}}(y_{t_{q,k}}) - 
	\sum_{q=1}^Q\sum_{k=1}^K F_{t_{q,k}}(y_{t_{q,k}}) \\
	&\quad + \sum_{q=1}^Q\sum_{k=1}^K 
	f_{t_{q,k}}(Y_{t_{q,k}})\mathds{1}_{Y_{t_{q,k}} \in \matroid} \\
	&\geq \sum_{q=1}^Q\sum_{i=K+1}^L F_{t_{q,i}}(y_{t_{q,i}}) + 
	\sum_{q=1}^Q\sum_{k=1}^K F_{t_{q,k}}(y_{t_{q,k}}) - 
	\sum_{q=1}^Q\sum_{k=1}^K M_1 + \sum_{q=1}^Q\sum_{k=1}^K 0 \\
	&= \sum_{t=1}^T F_{t}(y_{t}) - QKM_1,
	\end{split}    
	\end{equation}
	where the second equation holds since for $ t \in \{(q-1)L+1,\cdots, qL\} 
	\setminus \{ t_{q,1}, \cdots, t_{q,K} \} $, $Y_t = \losslessround(x_q) \in 
	\matroid$, and the inequality holds because of the fact that the rounding 
	is lossless and \cref{assump_on_discrete_bandit_f}.
	
	Therefore, by \cref{eq:discrete_regret,eq:discrete_aux1,eq:discrete_aux2} 
	and the optimality of $x^*$, we have
	\begin{equation}
	\label{eq:discrete_regret_final}
	\mathcal{R}_T  \leq \sum_{t=1}^T[(1-1/e)F_t(\tilde{x}) - F_{t}(y_{t})] + 
	QKM_1 \leq \sum_{t=1}^T[(1-1/e)F_t(x^*) - F_{t}(y_{t})] + QKM_1.    
	\end{equation}
	
	Now we can repeat the proof of \cref{thm:bandit} (\cref{app:general_K}) to 
	upper bound $\sum_{t=1}^T[(1-1/e)F_t(x^*) - \sum_{t=1}^T F_{t}(y_{t})$, 
	with $L_1 = 2M_1\sqrt{d}, L_2 = 4M_1\sqrt{d(d-1)}$ by 
	\cref{lem:discrete_to_continuous}. The only difference is when we turn to 
	bound $\expect[\Delta_q^{(k)}]= \expect [\|\nabla 
	\bar{F}_{q,k-1}(x_q^{(k)}) - d_q^{(k)} \|^2]$, where $\bar{F}_{q,k}(x) = 
	\frac{\sum_{i=k+1}^L\hat{F}_{t_{q,i}}(x)}{L-k}$, we have a larger upper 
	bound for $\expect [\| \nabla \bar{F}_{q,k-1}(x_q^{(k)}) - g_{q,k} \|^2]$, 
	where $g_{q,k} = \frac{d}{\delta}f_{t_{q,k}}(Y_{t_{q,k}})u_{q,k}$. 
	Precisely, we have
	\begin{equation*}
	\begin{split}
	&\expect [\| \nabla \bar{F}_{q,k-1}(x_q^{(k)}) - g_{q,k} \|^2] \\
	=& \expect[\| \nabla \bar{F}_{q,k-1}(x_q^{(k)}) - \nabla 
	\hat{F}_{t_{q,k}}(x_q^{(k)}) + \nabla 
	\hat{F}_{t_{q,k}}(x_q^{(k)}) 
	-\frac{d}{\delta}F_{t_{q,k}}(y_{t_{q,k}})u_{q,k} + 
	\frac{d}{\delta}F_{t_{q,k}}(y_{t_{q,k}})u_{q,k} - g_{q,k} \|^2]  \\
	=& \expect[\expect[\| \nabla \bar{F}_{q,k-1}(x_q^{(k)}) - \nabla 
	\hat{F}_{t_{q,k}}(x_q^{(k)}) \|^2  + \|\nabla \hat{F}_{t_{q,k}}(x_q^{(k)}) 
	- \frac{d}{\delta}F_{t_{q,k}}(y_{t_{q,k}})u_{q,k}\|^2 \\ 
	&\quad +\| \frac{d}{\delta}F_{t_{q,k}}(y_{t_{q,k}})u_{q,k} - g_{q,k} \|^2 \\
	\leq& L_1^2 + \frac{d^2M_1^2}{\delta^2} + \frac{d^2M_1^2}{\delta^2}\\
	=& L_1^2 + \frac{2d^2M_1^2}{\delta^2} \\
	\triangleq& \sigma^2.    
	\end{split}    
	\end{equation*}
	
	Plug in the new upper bound for $\sigma^2$, and repeat the analysis of 
	\cref{thm:bandit}, we have \begin{equation}
	\label{eq:discrete_aux3}
	\expect[\sum_{t=1}^T[(1-1/e)F_t(x^*) - F_{t}(y_{t})]] \leq N 
	T^{\frac{8}{9}} + 
	\frac{3r[2L_1^2+(3L_2R+2L_1)^2]}{4^{1/3}(\sqrt{d}+2)}T^{\frac{2}{3}} + 
	\frac{L_2D^2}{2}T^{\frac{1}{3}},
	\end{equation}
	where $N=\frac{(1-1/e)r}{\sqrt{d}+2}[\sqrt{d}(\frac{R}{r}+1)+\frac{R}{r}] 
	L_1 + \frac{(2-1/e)r}{\sqrt{d}+2}L_1+ 
	2M_1+\frac{3\cdot4^{2/3}(\sqrt{d}+2)d^2M_1^2}{r}+\frac{3(\sqrt{d}+2)D^2}{4r}+C$,
	 $C$ is a constant satisfying $\mathcal{R}_Q^{\mathcal{E}} \leq C\sqrt{Q}$. 
	
	Combine \cref{eq:discrete_regret_final,eq:discrete_aux3}, and using 
	$QKM_1=M_1T^{8/9}, D \leq \sqrt{d}, R \leq \sqrt{d}$, we conclude
	\begin{equation*}
	\expect[\mathcal{R}_T] \leq N T^{\frac{8}{9}} + 
	\frac{3r[2L_1^2+(3\sqrt{d}L_2+2L_1)^2]}{4^{1/3}(\sqrt{d}+2)}T^{\frac{2}{3}} 
	+ \frac{L_2d}{2}T^{\frac{1}{3}}, \end{equation*}
	where $N=\frac{(1-1/e)r}{\sqrt{d}+2}[\frac{d}{r}+\sqrt{d}(1+\frac{1}{r})] 
	L_1 + \frac{(2-1/e)r}{\sqrt{d}+2}L_1+ 
	3M_1+\frac{3\cdot4^{2/3}(\sqrt{d}+2)d^2M_1^2}{r}+\frac{3(\sqrt{d}+2)d}{4r}+C$,
	 $C$ is a constant satisfying $\mathcal{R}_Q^{\mathcal{E}} \leq C\sqrt{Q}$.
\end{proof}

\end{document}